\documentclass[final,12pt]{arxiv_temp} % Include author names

\usepackage[utf8]{inputenc} % allow utf-8 input
\usepackage[T1]{fontenc}    % use 8-bit T1 fonts
\usepackage{booktabs}       % professional-quality tables
\usepackage{amsfonts}       % blackboard math symbols
\usepackage{nicefrac}       % compact symbols for 1/2, etc.
\usepackage{microtype}      % microtypography
\usepackage{xcolor}         % colors
\usepackage{array}
\usepackage{bbm}
\usepackage{enumitem}
\usepackage{mathtools}
\usepackage{nicematrix}
\usepackage{bm}
\usepackage{tikz}
\usetikzlibrary{positioning}
\usetikzlibrary{bayesnet}
\usetikzlibrary{fit}
\usetikzlibrary{shapes,arrows}
\usetikzlibrary{shapes.multipart}
\usetikzlibrary{arrows.meta}
\usetikzlibrary{fit}
\usepackage{breakcites}
\usepackage{hyperref}
\usepackage[capitalise]{cleveref}
\usepackage{threeparttable}
\usepackage{graphicx}
\usepackage{float}
\makeatletter

\renewcommand*{\@opargbegintheorem}[3]{\trivlist
  \item[\hskip \labelsep{\bfseries #1\ #2}] \textbf{(#3)}\ \itshape}
\makeatother

\def\*#1{\mathbf{#1}}
\def\pa#1{\text{Pa}#1}
\def\nd#1{\text{Nd}#1}

\newcommand{\ind}{\perp\!\!\!\perp}
\newcommand{\nind}{\not\!\perp\!\!\!\perp}

\newcommand{\cG}{{\mathcal{G}}}
\newcommand{\cH}{{\mathcal{H}}}

\newcommand{\cP}{{\mathcal{P}}}

\newcommand{\cB}{{\mathcal{B}}}

\newcommand{\cW}{{\mathcal{W}}}

\newcommand{\boldeta}{\boldsymbol{\eta}}
\newcommand{\integers}{\mathbb{Z}}
\newcommand{\gts}{\mathcal{G}_{ts}}
\newcommand{\tmax}{\tau_{\max}}
\newcommand{\gw}{\cG_\cW}
\newcommand{\hw}{\cH_\cW}
\newcommand{\sw}{\*S_{\cW}}
\newcommand{\pre}{\text{Pre}}

\newtheorem{assumption}{Assumption}
\newcommand{\assumptionref}[1]{Assumption~\ref{#1}}

\usepackage{tikz}
\usetikzlibrary{arrows.meta,positioning,shapes,shapes.geometric,fit,backgrounds}

\makeatletter
\def\mathcolor#1#{\@mathcolor{#1}}
\def\@mathcolor#1#2#3{%
  \protect\leavevmode
  \begingroup
    \color#1{#2}#3%
  \endgroup
}
\makeatother
\makeatletter

\title[Learning Causal Structure of Time Series using Best Order Score Search]{Learning Causal Structure of Time Series using Best Order Score Search} %\vspace{0.5cm}}
\usepackage{times}

\arxivauthor{%
 \Name{Irene Gema Castillo Mansilla} \Email{irene.castillo@uni-potsdam.de}\\
 \addr University of Potsdam, Potsdam, Germany
 \AND
 \Name{Urmi Ninad} \Email{urmi.ninad@uni-potsdam.de}\\
 \addr Department of Computer Science, University of Potsdam, Potsdam, Germany
}

\begin{document}

\maketitle

\begin{abstract}%
    Causal structure learning from observational data is central to many scientific and policy domains, but the time series setting common to many disciplines poses several challenges due to temporal dependence. In this paper we focus on score-based causal discovery for multivariate time series and introduce TS-BOSS, a time series extension of the recently proposed Best Order Score Search (BOSS) \citep{BOSS}. TS-BOSS performs a permutation-based search over dynamic Bayesian network structures while leveraging grow–shrink trees to cache intermediate score computations, preserving the scalability and strong empirical performance of BOSS in the static setting. We provide theoretical guarantees establishing the soundness of TS-BOSS under suitable assumptions, and we present an intermediate result that extends classical subgraph minimality results for permutation-based methods to the dynamic (time series) setting.  This result paves the way for time series extensions of sparsest permutation-based causal discovery algorithms under suitable stationarity assumptions. We present experimental results on synthetic data that highlight the high accuracy of TS-BOSS.
\end{abstract}


\section{Introduction}
% \begin{itemize}
%     \item Overview CD, constraint and score based, BIC
%     \item time series CD
%     \item BOSS
%     \item Our contribution
% \end{itemize}
Learning the causal structure underlying observational data is fundamental across a wide range of disciplines, from policy-relevant fields such as economics and epidemiology to natural sciences including neuroscience and Earth system science. 
Causal structure learning methods can be broadly categorized into constraint-based and score-based methods. The former make use of conditional independence constraints learned from the data to infer causal relationships, while the latter optimize a consistently-defined score defined over the space of directed acyclic graphs (DAGs). Further methods exist such as restricted structural causal model (SCM) framework which employs assumptions on functional relationships and noise distributions to uncover causal relationships. See \citep{PearlCausality, Spirtes2000CPS, peters2017elements} for details on the different approaches and \citep{Glymour2019Review, squires_review} for a review. 

Often, the observational data take the form of time series, with observations at each time point influenced by past observations. The challenges specific to time series causal structure learning, or time series \emph{causal discovery}, have been discussed at length \citep{runge_inferring_2019, Camps_Valls_phys_reps}. For instance, in the setting of a single multivariate time series, temporal dependence violates the i.i.d.~assumption, thereby limiting the theoretical guarantees of causal discovery methods and statistical procedures more generally. Several methods spanning the different approaches for causal discovery have been proposed for the time series setting \citep{Hyvarinen_varlingam, runge2020discovering, dynotears}, see \citep{Assaad_survey} for a survey. 

In this work, our focus will be the score-based approach to causal discovery, while presenting experimental comparisons to established algorithms for time series causal discovery. Recently, \emph{best order score search} (BOSS), a score-based method for structure learning, was introduced \citep{BOSS}. BOSS performs a permutation-based search over DAGs and leverages a novel data structure—grow–shrink trees (GSTs)—to cache intermediate computations. BOSS was shown to achieve high accuracy and state-of-the-art performance while maintaining scalability to a large number of nodes. We propose a time series extension of BOSS for causal discovery over multivariate time series data named \textbf{TS-BOSS}. We provide theoretical results for the soundness of TS-BOSS and ablation studies that compare TS-BOSS to standard algorithms for time series causal discovery. Our results show that TS-BOSS achieves a better overall precision--recall trade-off than the competing methods, especially in the high autocorrelation regime where high sample dependence would be expected to hamper performance. 

This establishes TS-BOSS as a high-performing and scalable time series causal discovery method. We also provide an intermediate result (\theoremref{thm:window_Verma}) that extends the result of \citep{verma_pearl_networks} (Theorem 2 and Corollary 2) to the time series (i.e.~dynamic bayesian network) setting. This result has the useful implication that sparsest permutation-based methods for causal discovery \citep{raskutti_sp}
can be extended to the time series setting under suitable stationarity assumptions. 

The remainder of this paper is organized as follows: In \sectionref{sec:related} we present the related work, and in \sectionref{sec:prelim} we provide the general background for the time series causal discovery problem. In \sectionref{sec:tsboss} we present our algorithm TS-BOSS and in \sectionref{sec:theory} we provide theoretical results that guarantee soundness of TS-BOSS. In \sectionref{sec:experiments} we present a simulation study on synthetic data and conclude with a discussion in \sectionref{sec:outlook}. 

\section{Related work}\label{sec:related}
%GES, SP, GRaSP, BOSS, PCMCI+, DYNOTEARS, NOTEARS, DAGMA, \cite{ghahramani_DBN}, SVAR-GES, \cite{wienobst_discrete}

Greedy Equivalence Search (GES) was an influential early method in score-based causal discovery which leveraged Meek's conjecture to search greedily over equivalence classes of DAGs \citep{Chickering_ges}. 
Causal discovery was presented as a search problem over permutations that yield the sparsest DAG in \citep{raskutti_sp} with the \emph{sparsest permutation} (SP) algorithm. \citet{Solus_GSP, lam_grasp} proposed greedy score-based approaches for the permutation search over DAGs and \citet{BOSS} proposed an efficient algorithm within the same class of methods using Grow-Shrink trees. NOTEARS algorithm of \citet{Zheng_NOTEARS} presented the search over space of DAGs as a continuous optimization problem, and \citet{dynotears} extended this work to the dynamic setting with DYNOTEARS. Several other continuous optimization-based methods and perspectives thereon have been proposed since then \citep{Bello_DAGMA, reisach_varsortability, ng_sober_look}. Recently, \citet{wienobst_discrete} proposed a discrete-search based method using a modified grow-shrink procedure that provides significant speed-up as graph size increases.
 
%For general structure learning in the time series case, \cite{ghahramani_DBN} presents a review of Bayesian model selection theory and practice for dynamic Bayesian networks.
In the dynamic setting, \citet{runge2020discovering} proposed a constraint-based causal discovery algorithm for time series using a modified conditional independence (CI) test to mitigate the effects of ill-calibrated CI tests in the single multivariate time series setting while allowing for contemporaneous edges. In \citep{malinsky_svar_ges}, a hybrid constraint and score-based method for time series causal discovery with latent confounders was introduced, that uses GES as a model selection step. %For general structure learning in the time series case, \cite{ghahramani_DBN} presents a review of Bayesian model selection theory and practice for dynamic Bayesian networks. 

\section{Time Series Causal Discovery}\label{sec:prelim}

\subsection{Preliminaries and Problem Setup}

Let $\*{X}_t = (X_t^1,\dots,X_t^m)$ denote a vector of $m$ random variables
observed at discrete time points $t \in \integers$.
The collection $(\*{X}_t)_{t \in \integers}$ is called a
\emph{multivariate time series}. We assume that dynamical processes that are representable as multivariate time series can be modeled as a \emph{time series SCM} (or ts-SCM):
\begin{equation}\label{eq:tsSCM}
    X^j_t := f^j(\pa(X^j_t), \eta^j_t) \qquad \forall  \ X^j_t \in \*X_t \ \text{ and } \  \forall \ t \in \integers \ ,
\end{equation}
with `$:=$' denoting structural assignments. Here, each $f^j$ denotes the causal mechanism corresponding to $X^j_t$ that models each $X^j_t$ in terms of the parents of $X^j_t$, represented by $\pa(X^j_t)$, and the random variable $\eta^j_t$ which is the noise associated to $X^j_t$. For each $t \in \integers$, the noises $\boldeta_t = (\eta_t^1, \dots,\eta_t^m)$ are assumed to be jointly independent. Furthermore, for each $i\in [m]$, the noises $\boldeta^i = (\eta_t^i)_{t\in\integers}$ are assumed to be i.i.d.. The time series SCM \eqref{eq:tsSCM} entails a \emph{time series graph} with nodes $X^i_t$ for all $i \in [m]$ and for all $t \in \integers$, and directed edges from a node $X^i_t$ to a node $X^j_{t'}$ if and only if $X^i_t \in \pa(X^j_{t'})$, for all $i,j \in [m]$ and $t,t' \in \integers$ such that $t\leq t', (i,t)\neq (j,t')$. The latter condition excludes self-cycles in the time series graph, and we additionally exclude general directed cycles. The resultant graph is referred to as a \emph{time series directed acyclic graph} or a \emph{ts-DAG} \citep{Gerhardus_AoS}. The concept of d-separation \citep{PearlCausality} of two disjoint sets of vertices given a third extends naturally to ts-DAGs if the sets are finite. Given finite disjoint subsets $\*S_1, \*S_2, \*S_3$ of $(\*{X}_t)_{t \in \integers}$, we denote the statement $\*S_1$ is d-separated of $\*S_2$ given $\*S_3$ in a ts-DAG $\gts$ as $\*S_1 \ind_d \*S_2 \ | \  \*S_3$. DAGs with the same set of d-separations are represented by a \emph{completed partially directed acyclic graph} (CPDAG), and a CPDAG with additional edge orientations due to background knowledge, such as time order, is referred to as the \emph{maximally oriented partially directed acyclic graph} (MPDAG). Accordingly, we define the time series versions of these graphical objects namely ts-CPDAG and ts-MPDAG, respectively. 

In this work, the focus is causal discovery of the ts-DAG $\gts$, entailed by SCM \eqref{eq:tsSCM} underlying the multivariate time series $(\*{X}_t)_{t \in \integers}$. We make the following assumptions relating graphical d-separation  in $\gts$ and (conditional) independence statements of a finite collection of random variables $X^i_t$ to enable causal discovery. Following the notation of \citet{hochsprung_markov}, we denote finite disjoint subsets $I_1, I_2, I_3 \subset [m]\times \integers$ and corresponding finite disjoint subsets $\*S_m = \{X^i_t \ : \ (i,t) \in I_m\}$ for $m = \{1,2,3\}$ for the following assumptions.
% \begin{assumption}[Time series Global Markov property]\label{ass:markov}
%     For any finite disjoint sets $I_1, I_2, I_3 \subset [m]\times \integers$, if $\*S_1 \ind_d \*S_2 \ | \  \*S_3$ in $\gts$, then $\*S_1 \ind_d \*S_2 \ | \  \*S_3$. 
% \end{assumption}
\begin{assumption}[Time series local Markov property]\label{ass:localmarkov}
    For every $X^i_t \in (\*X_t)_{t \in \integers}$ and every finite set $\*S_i \subset (\*X_t)_{t \in \integers}$, such that $\*S_i$ contains no descendant or parent of $X^i_t$, $X^i_t \ind \*S_i | \pa(X^i_t)$, where parents and descendants are defined w.r.t.~the ts-DAG $\gts$.
\end{assumption}

\begin{assumption}[Time series faithfulness property]\label{ass:faith}
   For any finite disjoint sets $I_1, I_2$ and $I_3$, with
$I_j \subseteq [m]\times\mathbb{Z}$ for $j=1,2,3$,  if $\*S_1 \ind \*S_2 \ | \  \*S_3$, then $\*S_1 \ind_d \*S_2 \ | \  \*S_3$ in $\gts$.
\end{assumption}
Furthermore, the requirement that noises $\boldeta_t = (\eta_t^1, \dots,\eta_t^m)$ are jointly independent imposes \emph{causal sufficiency}, namely there are no latent confounders. For two nodes $X^i_s,X^j_t$ in $\gts$ such that $s \leq t$ and $X^i_s \in \pa(X^j_t)$ , the difference $\tau = t - s$ is referred to as the \emph{time lag} corresponding to the edge $X^i_s \to X^j_t$. %Consequently, all directed edges in $\gts$ have an associated time lag. % require the following definitions to state the \cref{ass:station}. % of \emph{causal stationarity}. 
% \begin{definition}[Time lag]
% For two time indices $s < t$, the difference $\tau = t - s$ is called a
% \emph{time lag}.
% \end{definition}
%We now state further assumptions common for time series causal discovery.
% Structural causal models (SCMs) can be extended to the time-series setting by
% allowing the structural equations of a variable at time $t$ to depend on
% variables at earlier time points.
% Accordingly, each variable $X_t^i$ may be modeled as a function of variables
% $\{X_{t-\tau}^j\}$ at different time lags and an associated exogenous noise term
% \cite{Runge2018CITS}.
In accordance with the time-series SCM formulation, causal relationships are
restricted to respect the temporal order of the data, such that variables at
future time points cannot be parents of variables at earlier time points. Note that we allow for zero time lags (i.e. contemporaneous edges), in order to account for insufficient time resolution of the time series.

\begin{assumption}[Maximum time lag]\label{ass:maxlag}
The time-series causal graph $\gts$ is assumed to have a known finite maximum time lag
$\tmax \geq 0$, such that every edge in $\gts$ has an associated
time lag $\tau \leq \tmax$.
\end{assumption}

\begin{assumption}[Stationary causal structure]\label{ass:station}
The causal structure of ts-SCM \eqref{eq:tsSCM} is \emph{stationary}, namely the causal mechanisms $f^i$ corresponding to $X^i_t$ in \equationref{eq:tsSCM} remain constant for all $t\in \integers$. %do not change over time.
% Consequently, causal relations depend only on relative time lags and not on absolute time indices.
\end{assumption}

\assumptionref{ass:station} implies that the ts-DAG $\gts$ entailed by ts-SCM \eqref{eq:tsSCM} has the \emph{repeating edge property} ($X^j_{t'-s'} \to X^i_{t'} \Leftrightarrow X^j_{t-s'} \to X^i_{t}$) because of which the edges with an arrowhead incident on $\*X_t$ for any $t \in \integers$ are sufficient to reconstruct the entire causal graph $\gts$ \citep{Gerhardus_AoS}. $\*X_t$, for an arbitrary $t$, will be referred to as the \emph{contemporaneous slice} of time series $(\*X_t)_{t \in \integers}$. By \assumptionref{ass:maxlag} and \ref{ass:station}, we can conclude that the following target object of time series causal discovery yields $\gts$ \citep{Assaad_survey}.% called the \emph{window causal graph} \citep{Assaad_survey}.

\begin{definition}[Window causal graph]\label{def:window_graph}
    Let $\tmax\geq 0$ be the maximum time lag ts-DAG $\gts$ and $\cW := [t-\tmax, t]$ be a time window for an arbitrary $t \in \integers$ over the set of discrete time indices. Then, the finite index set $I_\cW = [m] \times \cW$, results in the finite set of vertices $\sw = \{X^i_t \ : \ (i,t)\in I_\cW\}$ in the ts-DAG $\gts$. The induced graph over the set $\sw$ is defined as the window causal graph (or simply window graph) and denoted by $\gw$.
\end{definition}

% \subsection{Assumptions}

% \begin{definition}[Multivariate time series]
% Let $\mathbf{X}_t = (X_t^1,\dots,X_t^m)$ denote a vector of $m$ random variables
% observed at discrete time points $t \in \integers$.
% The collection $\{\mathbf{X}_t\}_{t \in \integers}$ is called a
% \emph{multivariate time series}.
% \end{definition}

\subsection{Causal Discovery with Single versus Multiple Multivariate Time Series}\label{seccausaldiscov}

Time series causal discovery is commonly studied in two formats: the \emph{multiple time series} setting, where multiple independent multivariate time series with same underling ts-SCM are observed, and the \emph{sliding-window} setting, where a single dependent multivariate series is converted into data over time windows with a known maximum lag \citep{runge2020discovering,manten25, wiecksosa_CIT_single}. In both cases, the central problem is identifying the parents of the contemporaneous slice in the resulting window graph, since \assumptionref{ass:station} ensures the reconstruction of the time series graph once the window graph is known. While consistency of the Bayesian scoring criterion (BSC), employed as the score in score-based causal discovery, in the multiple time series setting follows relatively directly from standard i.i.d.~arguments, it is more delicate in the sliding-window setting due to sample dependence, though recent works have established consistency for the Bayesian information criterion (BIC) under certain assumptions \citep{bardet2021efficient}. Here, we focus on this shared window-graph discovery problem while assuming score consistency (\assumptionref{ass:asymptotic}, formally stated in \sectionref{sec:theory}), and present an experimental comparison of the single versus multiple time series data with BIC score in our experimental analysis (\sectionref{sec:experiments}). Our simulations show that the proposed algorithm (\sectionref{sec:tsboss}) performs very similarly in both settings.  %and we point to existing results establishing BIC-type consistency for time series models.
% \begin{itemize}
%     \item Two formats for time series causal discovery: sliding window vs ensemble setting
%     \item Common problem in both formats: search for parents of contemp slice 
%     \item Score consistency follows trivially in ensemble case, less trivial in sliding window case 
%     \item Our approach: Focus on common problem of window graph discovery while assuming score consistency. Present references on BIC consistency for time series.
% \end{itemize}

\section{TS-BOSS}\label{sec:tsboss}

\subsection{Overview of BOSS}
% \begin{itemize}
%     \item Broadly explain score-based logic. Score consistency and decomposability imply local consistency.
%     \item Explain logic of permutation-based approaches within score-based methods. Instead searching over entire space of DAGs, search over permutation. Relies on subgraph minimality, i.e., each permutation has a subgraph minimal DAG associated to it.
%     \item BOSS phase 1. Permutation search using best position of a variable in a permutation, permutation to DAG using grow-shrink trees.
%     \item BOSS phase 2 = backward equivalence search (BES). 
% \end{itemize}

Score-based causal discovery infers a (equivalence class over a) DAG by optimizing a score that trades off goodness-of-fit and complexity; under score consistency and decomposability, improvements in the global score translate into \emph{locally consistent} decisions about adding or removing individual parent-child relationships (see Definition 5 and 6 in \citep{Chickering_ges} for details on score consistency). Permutation-based methods exploit this structure by searching over variable permutations rather than the full DAG space: for each permutation, one can associate a unique subgraph-minimal DAG consistent with that order (Corollary 2 in \citep{verma_pearl_networks}), reducing the search to identifying an optimal permutation and its induced minimal graph. BOSS operationalizes this in two stages: in Phase 1 it performs a permutation search by repeatedly finding a variable’s best position in the ordering and constructing the corresponding DAG efficiently via grow–shrink trees (illustrated in Appendix~\ref{app:growshrink}, Figures~\ref{fig:grow}~and~\ref{fig:shrink});  and in Phase 2 it applies a backward equivalence search (BES, also second phase of GES) to perform edge deletions that improve the score which guarantees asymptotic correctness.

\subsection{TS-BOSS Pseudocode and Adaptation Challenges}

In the following we present the adaptation idea behind TS-BOSS, and present a pseudocode in \algorithmref{alg:tsboss}.
We note here that the challenges of adaptation to time series and our proposed solution are common to most discrete-search score-based methods and thus their application is not limited to extensions of BOSS.
% \begin{itemize}
%     \item \un{TODO for Irene, Condense section 4.4,4.6,4.7 of project report} %Impose temporal constraints in search procedure (stationarity, max lag, causal order).
%     Mention that the challenges of adaptation to time series and our proposed solution are common to most discrete-search score-based methods,not just BOSS. 
%     % \item Define modified permutation-induced window causal graph (RU') corresponding to RU defined in \cite{lam_grasp}
%     \item \un{TODO for Irene} Pseudocode
%     \item  %Define TS-BES for window causal graph ($\boldsymbol{\cE}^-(\cE)$ corresponds to those DAGs whose parent sets for contemp nodes are fewer.)
% \end{itemize}
Let $(\mathbf{X}_t)_{t \in \mathbb{Z}}$ denote a stationary multivariate time series with maximum time lag $\tau_{\max}$. To recover the window causal graph up to its Markov equivalence class together with background knowledge of time order, namely the  ts-MPDAG of the window causal graph, TS-BOSS proceeds in two phases.

Figure~\ref{fig:tsboss_flowchart} provides a high-level overview of the TS-BOSS workflow; the formal pseudocode is given in Algorithm~\ref{alg:tsboss}.

% --- Figure ---
\begin{figure}[h]
\centering
\resizebox{\linewidth}{!}{%
\begin{tikzpicture}[
  font=\footnotesize,
  term/.style={draw=black!70, fill=gray!10, rounded rectangle,
               minimum height=8mm, text width=28mm, align=center, line width=0.6pt},
  proc/.style={draw=black!70, fill=gray!6, rectangle,
               minimum height=8mm, text width=30mm, align=center, line width=0.6pt},
  dec/.style={draw=black!70, fill=gray!6, diamond, aspect=2.3,
              text width=17mm, align=center, inner sep=1pt, line width=0.6pt},
  flow/.style={-{Latex[length=2mm, width=1.5mm]}, thick, color=black!75},
  yn/.style={font=\scriptsize, fill=white, inner sep=1pt},
  phase1/.style={draw=black!40, dashed, rounded corners=4pt,
                 fill=blue!4, inner sep=5pt},
  phase2/.style={draw=black!40, dashed, rounded corners=4pt,
                 fill=orange!5, inner sep=5pt},
  lbl/.style={font=\scriptsize\itshape}
]

%% Nodes
    \node[term] (start)   {Input:\\time-series data $\mathbf{D}$\\and $\tau_{\max}$};
\node[proc, right=6mm of start]  (unroll) {Unroll time series\\into lagged window representation};
\node[proc, right=5mm of unroll] (init)   {Initialize permutation\\and $\mathtt{GST}$ structures};
\node[proc, right=5mm of init]   (move)   {Evaluate candidate move\\with $\mathtt{BestTSMove}$};
\node[dec,  right=6mm of move]   (improv) {Score\\improves?};
\node[proc, below=12mm of improv](update) {Update\\permutation};
\node[proc, right=7mm of improv] (proj)   {Induce ts-DAG via $\mathtt{Project}$\\and convert with $\mathtt{TSDAGtoTSMPDAG}$};
\node[proc, right=5mm of proj]   (bes)    {Apply TS-BES\\refinement};
\node[term, right=5mm of bes]    (end)    {Output:\\estimated \\ts-MPDAG};

%% Arrows
\draw[flow] (start)  -- (unroll);
\draw[flow] (unroll) -- (init);
\draw[flow] (init)   -- (move);
\draw[flow] (move)   -- (improv);
\draw[flow] (improv.south) -- node[yn, right=1pt]{yes} (update.north);
\draw[flow] (update.west)  -| (move.south);
\draw[flow] (improv.east)  -- node[yn, above=1pt]{no}  (proj.west);
\draw[flow] (proj)   -- (bes);
\draw[flow] (bes)    -- (end);

%% Phase brackets (drawn behind nodes)
\begin{pgfonlayer}{background}
  \node[phase1, fit=(unroll)(init)(move)(improv)(update),
        label={[lbl]below:Phase 1: Score-based permutation search}] {};
  \node[phase2, fit=(proj)(bes),
        label={[lbl]below:Phase 2: Graph refinement via TS-BES}] {};
\end{pgfonlayer}

\end{tikzpicture}%
}
\caption{High-level workflow of TS-BOSS. \textit{Phase 1} performs score-based search over admissible contemporaneous permutations. \textit{Phase 2} applies TS-BES to refine the induced graph and obtain the final ts-MPDAG.}
\label{fig:tsboss_flowchart}
\end{figure}

\paragraph{First Phase of TS-BOSS.}
The first phase extends the permutation search and grow--shrink procedure of BOSS via the following modifications:

\begin{itemize}%[topsep=0pt,noitemsep, itemindent=0pt]

\item \textbf{Time-window unrolling.}  
For a fixed $\tau_{\max}$, the time series is unrolled into the variable set
\[
\{\mathbf{X}_{t-\tau_{\max}}, \dots, \mathbf{X}_t\}, \quad
\text{where } \mathbf{X}_t = (X_t^1,\dots,X_t^m).
\]
This yields a window graph with distinct nodes across time steps.
\item \textbf{Temporal order constraint in permutations} Temporal order is enforced in time series: lagged variables must precede contemporaneous variables. Using the notation of \citet{BOSS}, the lagged variables must be always placed in early positions than the contemporaneous, permutations take the form  \[
\pi
=
(
\underbrace{\mathbf{X}_{t-\tau_{\max}},\dots,\mathbf{X}_{t-1}}_{
\begin{array}{c} \scriptstyle
\text{fixed across permutations} 
\end{array}
},
\, \mathbf{X}_t
).
\]

\item \textbf{Restricted permutation search space.} 
Under stationarity (\assumptionref{ass:station}), the causal structure is invariant over time. Hence, only edges from lagged variables to $\mathbf{X}_t$ and edges within $\mathbf{X}_t$ must be determined. The permutation search is therefore restricted to contemporaneous variables: only contemporaneous variables are permuted and grow--shrink trees are constructed only for nodes in $\mathbf{X}_t$, while including lagged variables as candidate parents.
\end{itemize}

\begin{algorithm2e}
\caption{TS-BOSS (First Phase)}
\label{alg:tsboss}
 \DontPrintSemicolon
 \LinesNumbered
\KwIn{Time series data $\*D$, maximum lag $\tau_{\max}\ge 0$}
\KwOut{Estimated ts-MPDAG $\widehat{\gts}$}

$\textbf{D} \leftarrow \mathtt{Unroll}(\textbf{D}, \tau_{\max})$ \tcp*{Unroll time series data}
$\pi \leftarrow (\mathbf{X}_{t-\tau_{\max}}, \dots, \mathbf{X}_{t-1}, \mathbf{X}_t)$ \tcp*{Initialize permutation }
$\mathcal{T} \leftarrow \{\mathtt{GST}(X_t^i,\widehat{\mathcal{D}})\}_{i=1}^m$ \tcp*{Initialize GS trees for $\mathbf{X}_t$ }
$\mathrm{improved} \leftarrow \mathrm{True}$ \tcp*{Flag for permutation search}

\While{$\mathrm{improved}$}{
    $s_{\mathrm{best}} \leftarrow \mathcal{T}.\mathtt{Score}(\pi)$ \tcp*{Evaluate current score}
    $\mathrm{improved} \leftarrow \mathrm{False}$\;
    \For{$i\leftarrow 1$ \KwTo $m$}{
        $\pi' \leftarrow \mathtt{BestTSMove}(\mathcal{T}, \pi, X_t^i)$ \tcp*{Only contemporaneous variables}
        \If{$\mathcal{T}.\mathtt{Score}(\pi') > s_{\mathrm{best}}$}{
            $\pi \leftarrow \pi'$ \tcp*{Update permutation}
            $\mathrm{improved} \leftarrow \mathrm{True}$\;
            \textbf{break}\;
        }
    }
}
$\gts  \leftarrow \mathcal{T}.\mathtt{Project}(\pi)$ \tcp*{ts-DAG induced by $\pi$}
$\widehat{\gts} \leftarrow \mathtt{TSDAGtoTSMPDAG}(\gts)$ \tcp*{Compute ts-MPDAG}
\Return $\widehat{\gts}$\;
\end{algorithm2e}

We now formalize the first phase of TS-BOSS. 
Algorithm~\ref{alg:tsboss} presents the complete pseudocode, following 
the structure of BOSS \citep{BOSS} with the time-series adaptations 
described above. The algorithm relies on several auxiliary subroutines, 
which we briefly describe for completeness. The subroutine $\mathtt{Unroll}$ restructures the time series into a lag-unrolled representation, in which the variables $\{X_{t-\tau}^i\}_{\tau=0}^{\tau_{\max}}$ are treated as distinct nodes for scoring and permutation search. The grow--shrink procedure $\mathtt{GST}$ and the projection step $\mathtt{Project}$ follow the corresponding routines in BOSS \citep{BOSS}: for each contemporaneous node, $\mathtt{GST}$ incrementally builds and prunes a candidate parent set using local score improvements \footnote{The temporal ordering of lagged variables is fixed, whereas the ordering of contemporaneous variables is arbitrary (see line 2 in Algorithm~\ref{alg:tsboss}).}, and $\mathtt{Project}$ maps the resulting permutation and parent selections to the induced ts-DAG on the lag-unrolled window graph. The routine $\mathtt{BestTSMove}$ extends the BOSS $\mathtt{best\_move}$ subroutine by restricting the permutation search to admissible contemporaneous positions (see Algorithm~\ref{alg:besttsmove}). Finally, the subroutine $\mathtt{TSDAGtoTSMPDAG}$ constructs the ts-MPDAG from the induced ts-DAG by applying the procedure described in Appendix~\ref{app:tscpdag}. Since there are finitely many admissible permutations and each accepted move strictly improves the score, Algorithm~\ref{alg:tsboss} always terminates.

\paragraph{Second Phase of TS-BOSS} In order to prove that the resulting window graph is the most minimal graph that describes the data (i.e.~it fulfills \assumptionref{ass:faith} with respect to the time series), the first phase of TS-BOSS must be supplemented with a time series adaptation of the backward equivalence search (TS-BES). TS-BES operates analogously to BES in the static setting with the difference that the neighboring window graphs in the search space correspond to those whose parent sets for the nodes in the contemporaneous slice are fewer.

\begin{algorithm2e}
\caption{Subroutine: BestTSMove}
\label{alg:besttsmove}
 \DontPrintSemicolon
 \LinesNumbered
\KwIn{Grow--shrink trees $\mathcal{T}$, permutation $\pi$, variable $X_t^i$}
\KwOut{Updated permutation $\pi$}

$pos \leftarrow m \cdot \tau_{\max}$ \tcp*{Starting index of contemporaneous block}
$s_{\mathrm{best}} \leftarrow \mathcal{T}.\mathtt{Score}(\pi)$\;
$j \leftarrow \pi.\mathtt{index}(X_t^i)$ \tcp*{Current position of $X_t^i$}

\For{$k\leftarrow 1$ \KwTo $m$}{
    $\pi_{temp} \leftarrow \pi.\mathtt{move}(X_t^i, pos + k)$ \tcp*{Move within contemporaneous block}
    \If{$\mathcal{T}.\mathtt{Score}(\pi_{temp}) > s_{\mathrm{best}}$}{
        $s_{\mathrm{best}} \leftarrow \mathcal{T}.\mathtt{Score}(\pi_{temp})$\;
        $j \leftarrow pos + k$\;
    }
}
$\pi \leftarrow \pi.\mathtt{move}(X_t^i, j)$ \tcp*{Place at best position}
\Return $\pi$\;
\end{algorithm2e}

\section{Theoretical Results}\label{sec:theory}
We now present results that guarantee correctness of TS-BOSS in the large sample limit. %population-level correctness of TS-BOSS.
\begin{assumption}\label{ass:asymptotic}
    We assume access to:
    \begin{itemize}%[topsep=0pt,noitemsep]
        \item[(i)] a locally consistent Bayesian scoring criterion $\cB$ for DAGs and, %essentially it is the local consistency of the score that implicityl implies that the true ts-DAG is marov and faithful, because the score tries to find the "simplest" DAG that fits the distribution
        \item[(ii)] i.i.d.~data $\*D$ for nodes $\sw$ in the window graph $\gw$, see \definitionref{def:window_graph}.
    \end{itemize}
\end{assumption}
\citet{Haughton1988} has shown that for distributions that are in the curved exponential family, eg.~multivariate Gaussian distributions, the Bayesian scoring criterion equals the Bayesian information criterion (BIC) plus a constant error term that becomes insignificant in the large sample limit. We use the BIC score in our TS-BOSS implementation.

In order to employ permutation-search over window causal graph, we will now introduce a few definitions for the time series setting. In the following, we refer to a permutation $\pi$ over $\sw$ as an \emph{admissible permutation} if it respects time order, i.e., for the $i^{th}$ vertex in $\pi$ called $\pi_i \in \sw$, all vertices that precede $\pi_i$ in $\pi$, denoted by $\pre(i,\pi) = \{\pi_j \ : 1 \leq j < i\}$, have time indices less than or equal to the time index of $\pi_i$. Furthermore, the joint probability distribution over the variables $\sw$ is denoted by $P_\cW$, and it satisfies the \emph{graphoid axioms} \citep{verma_pearl_networks}.

\begin{definition}[Permutation-induced window graph]\label{def:RUprime}
  %Let $\sw$ denote the variables in a window graph $\gw$.
  Let $P_\cW$ be a graphoid over variables $\sw$ associated with the window graph $\gw$. Each admissible permutation $\pi$ of $\sw$ induces a DAG $\gw^\pi$ over $\sw$ defined as follows:
  \begin{itemize}%[topsep=0pt,noitemsep]
      \item[(i)] Let $t$ be the maximal time index of all the variables in $\pi$. For each $X^i_{t}\in\sw$ and each $X^j_{s} \in \pre(X^i_{t}, \pi)$, if $X^i_{t} \nind X^j_s \ | \  \pre(X^i_{t}, \pi) \setminus\{X^j_s\}$, then $ X^j_s \to X^i_{t}$ in $\gw^\pi$,
      \item[(ii)] For all variable indices $i,j$, time indices $t'<t$ in $\pi$ and $s' \geq 0$, there exists an edge $ X^j_{t'-s'} \to X^i_{t'}$ in $\gw^\pi$ if and only if $ X^j_{t-s'} \to X^i_{t}$ was found in (i) above.
  \end{itemize}
\end{definition}
Condition (i) in \definitionref{def:RUprime} constructs the edges incident into the variables at the maximal time index (also known as the `contemporaneous slice' in the window graph), and is analogous to the method discussed in \citep{raskutti_sp}. Condition (ii) is needed in order to satisfy time-shift invariance (\assumptionref{ass:station}) that is implicit in the definition of a valid window graph. We use the term \emph{stationary graph} for any window graph that satisfies condition (ii) in \definitionref{def:RUprime}.
\begin{definition}[Window Markov property]\label{def:window_Markov}
    %Let $P_\cW$ be a graphoid over $\sw$ and $\gw^\pi$ be a window graph induced from an admissible permutation $\pi$. $(P_\cW, \gw^\pi) $
     A graphoid $P_\cW$ and a window graph $\gw$ over variables $\sw$ is said to satisfy the window Markov property if for all $X^i_t \in \sw$,  $X^i_t \ind \nd(X^i_t)_{\gw}|\pa(X^i_t)_{\gw}$ for all $i$ and the maximal time index $t$ in $\gw$. Here $\nd(\cdot)_{\gw}$ refers to the set of non-descendants in the graph $\gw$. 
\end{definition}
From \assumptionref{ass:localmarkov}, \assumptionref{ass:maxlag}, \assumptionref{ass:station} and  \definitionref{def:window_Markov} it follows that if a ts-DAG $\gts$ satisfies the local Markov property with respect to a multivariate time series, then its corresponding window graph $\gw$ satisfies the window Markov property with respect to the graphoid induced over $\sw$ and vice-versa. % vice-versa because no parent can lie beyond the maximal lag, so all parents have been found when window Markov is satisfied and then local Markov holds because everything in the past is a non-descendant. 
\begin{definition}[Window subgraph minimality]\label{def:window_SGS}
    Let a graph $\gw$ over vertices $\sw$ satisfy the window Markov property w.r.t.~a graphoid $P_\cW$. Then $\gw$ is said to be window subgraph minimal if no stationary subgraph of $\gw$ satisfies the window Markov property.
\end{definition}
For an overview on subgraph minimality (also known as the minimal I-MAP characterization) in the non-time series setting, see \citep{pearl_probabilistic_reasoning_1988, verma_pearl_networks, raskutti_sp}. %Subgraph minimality plays a key role in sparsest permutation-search based algorithms such as \citep{raskutti_sp}, which searches over permutations while assigning to each permutation a DAG that is subgraph minimal and finally outputs the sparsest DAG across permutations whose soundness is guaranteed using a strictly weaker counterpart to faithfulness. BOSS uses a similar idea of assigning to a permutation its subgraph minimal DAG in phase 1 and improves the permutation search by local changes (best-move) and caching results ; phase 2 is required for theoretical soundness since all permutations are not guaranteed to be searched in phase 1. \cite{verma_pearl_networks} have proven that given a permutation, a DAG where for each node, its parent set is the smallest subset of its predecessors that renders it conditionally independent of the remaining predecessors satisfies the local Markov property and no proper subgraph thereof does (see Theorem 2 and Corollary 2 in their paper). Due to the connection of this result to any sparsest permutation-based algorithm (including BOSS), in order to extend BOSS to time series, we will require a similar result in the time series setting. \assumptionref{ass:station} implies if a window graph satisfies window subgraph minimality then the corresponding ts-DAG satisfies subgraph minimality with respect to time series local Markov property. We now extend the subgraph minimality result of \cite{verma_pearl_networks}%, which plays a key role in sparsest permutation-search based algorithms such as \citep{raskutti_sp}, 
%to the stationary time series setting.
Verma and Pearl (1990) proved that, given a permutation, the DAG in which each node's parent set is the smallest subset of its predecessors rendering it conditionally independent of the remaining predecessors satisfies the local Markov property, and that no proper subgraph thereof does (see their Theorem 2 and Corollary 2), also called \emph{subgraph minimality}. Subgraph minimality plays a key role in sparsest-permutation-search algorithms such as that of \cite{raskutti_sp}, which searches over permutations, assigns to each permutation a subgraph-minimal DAG, and outputs the sparsest such DAG across permutations; its soundness is guaranteed under a strictly weaker counterpart to faithfulness. BOSS uses a similar idea, assigning to each permutation its subgraph-minimal DAG in phase 1 and improving the permutation search through local changes (best-move) and result caching; phase 2 is required for theoretical soundness, since phase 1 does not guarantee that all permutations are searched.
Because the subgraph minimality result underlies any sparsest-permutation-based algorithm (including BOSS), extending BOSS to time series requires an analogous result in the time-series setting. \assumptionref{ass:station} implies that if a window graph satisfies window subgraph minimality, then the corresponding ts-DAG satisfies subgraph minimality with respect to the time-series local Markov property. Thus, we now extend the subgraph minimality to the stationary time-series setting.
\begin{theorem}[Permutation-induced window graph minimality]\label{thm:window_Verma} Let $P_\cW$ be a graphoid over $\sw$ and $\gw^\pi$ be a window graph induced from an admissible permutation $\pi$. Then $\gw^\pi$ satisfies the window Markov property and is window subgraph minimal. 
\end{theorem}
All proofs are relegated to Appendix \ref{app:proofs}. In the following result, subgraph minimality for a ts-DAG refers to minimality w.r.t.~the time series local Markov property, see explanation after \definitionref{def:window_SGS}.

\begin{lemma}
    %Let $\cB(\cG,\*D)$ be a locally consistent Bayesian scoring criterion for DAG $\cG$ and i.i.d.~window data $\*D$ over the vertices $\sw$.
    Let assumptions \ref{ass:localmarkov}-\ref{ass:asymptotic} be satisfied for data $\*D$ and ts-DAG $\gts$ entailed by a ts-SCM over time series $(\*X_t)_{t\in \integers}$, and a Bayesian scoring criterion $\cB(\cG,\*D)$ over DAGs $\cG$.
    Then the output graph of phase 1 of TS-BOSS over data $\*D$ with score  $\cB(\cG,\*D)$ satisfies time series local Markov property w.r.t.~$(\*X_t)_{t\in \integers}$ and is subgraph minimal in the large sample limit.
    % Then thw phase 1 of TS-BOSS returns the MEC of the true window causal graph $\gw$ in the large sample limit.
\end{lemma}
% \begin{proof}
% Suppose not, then the output ts-DAG $\widehat{\gts}$ of phase 1 of TS-BOSS is such that there exists an $X^j_{s}\in \nd(X^i_t)_{\widehat{\gts}}$ such that $X^i_t \nind X^j_{s} \ | \ \pa(X^i_t)_{\widehat{\gts}} $. From assumptions \ref{ass:localmarkov}, \ref{ass:maxlag} and \ref{ass:station}, it follows that there exists a that is a non-descendant in the window graph violates the window Markov property. W.l.o.g.~let this variable be $X^j_{s}$. Consider the last permutation in phase 1 before the permutation search terminates because no score improvement can be found. Local score consistency of $\cB(\cG,\*D)$, i.e.~\cref{ass:asymptotic},  implies that the ts-DAG that results from adding edge $X^j_{s} \to X^i_t$ when computing the permutation-induced graph, has a higher score than for a ts-DAG without this edge, therefore the first phase cannot terminate at $\widehat{\gts}$. 
% \end{proof}

This result ensures that phase 1 of TS-BOSS yields the sparsest window graph that is Markovian to the distribution, thus avoiding dense local optima, which is a problem common to score-based algorithms such as GES. Finally, the following result guarantees the asymptotic correctness of TS-BOSS.
\begin{lemma}
    Let assumptions \ref{ass:localmarkov}-\ref{ass:station} be satisfied for data $\*D$ and ts-DAG $\gts$ resulting from a ts-SCM over time series $(\*X_t)_{t\in \integers}$. 
    Let $\widehat{\cG}_{ts}$ be a ts-DAG that satisfies the time series local Markov property w.r.t.~$(\*X_t)_{t\in \integers}$. %corresponding to a multivariate time series that satisfies \cref{ass:localmarkov}.
    Then TS-BES with input graph $\widehat{\cG}_{ts}$, and a Bayesian scoring criterion $\cB(\cG,\*D)$ and time series data $\*D$ that satisfies \assumptionref{ass:asymptotic}, returns the ts-MPDAG corresponding to the true ts-DAG $\gts$ in the large sample limit.
\end{lemma}

\section{Simulation Study}\label{sec:experiments}
% We assess performance via controlled simulations, detailing the data generation, methods, evaluation metrics and experimental setup. 

\subsection{Data Generation}\label{sec:datagen}

Synthetic data are generated from multivariate linear ts-SCMs of the form

\[
X_t^j := \sum_{X^i_{t-\tau}\in \mathrm{Pa}(X_t^j)}
a_{ij}^{(\tau)} X^i_{t-\tau}
+ \eta_t^j,
\]

where \(i,j \in [m]\), \(\tau \in \{0,\dots,\tau_{\max}\}\), \(a_{ij}^{(\tau)} \in \mathbb{R}\), and \(\eta_t^j\) are mutually independent Gaussian noise terms with zero mean. The contemporaneous graph (\(\tau=0\)) is acyclic, while positive-lag self-links allow temporal dependence. Non-stationary parameterizations are discarded, and a burn-in period is removed before collecting samples from the stationary regime.

% Synthetic data are generated from multivariate linear ts-SCMs. Each variable follows a structural equation of the form
% \[
% X^j_t := \sum_{X^i_{t-\tau}\in \mathrm{Pa}(X^j_t)} a_{ij}^{(\tau)} X^i_{t-\tau}
% + \eta^j_t,
% \]
% where \(i,j \in [m]\), \(\tau \in \{0,\dots,\tau_{\max}\}\) denotes the time lag, \(a_{ij}^{(\tau)} \in \mathbb{R}\) and $\boldeta_t = (\eta_t^1, \dots,\eta_t^m)$ are mutually independent Gaussian noise terms with zero mean.
% %Causal effects act forward in time (\(\tau \ge 1\)), while t
% The contemporaneous graph (induced graph for nodes with \(\tau = 0\)) is acyclic. Self-links at positive lags permit temporal dependence.
% Non-stationary parameterizations are discarded. A burn-in period is simulated and removed to ensure convergence to the stationary regime before collecting the final sample.

\subsection{Methods}\label{sec:models}

We compare the following methods: (i) \textbf{TS-BOSS}: the proposed permutation-based score method adapted to time series by enforcing temporal ordering constraints in the sliding-window setting (see \sectionref{seccausaldiscov}); (ii) \textbf{TS-BOSS (i.i.d.)}: TS-BOSS applied in the multiple time series setting (see \sectionref{seccausaldiscov}). (iii) \textbf{PCMCI+} \citep{runge2020discovering}: a constraint-based method for time-series causal discovery (here applied with the partial correlation test for conditional independence). (iv) \textbf{DYNOTEARS} \citep{dynotears}: a score-based continuous optimization method based on linear SCMs with differentiable acyclicity constraints on the contemporaneous structure. (v) \textbf{TS-FGES}: a temporal extension of FGES \citep{ramsey2017million,Chickering_ges}, a more optimized variant of GES \citep{Chickering_ges}, implemented in TETRAD using background knowledge constraints to enforce temporal ordering among variables. For TS-BOSS methods, we present simulations for phase 1, as in \citep{BOSS}. PCMCI+, DYNOTEARS and TS-FGES are included for comparison.

\subsection{Experimental Setup and Evaluation Metrics}\label{sec:evalmetrics}

We report adjacency precision and recall, orientation precision and recall, and runtime. Since the estimated graphs are ts-MPDAGs, the ground-truth ts-DAG is first converted into the corresponding ts-MPDAG before evaluation (\appendixref{app:tscpdag}). Details on the extension of the adjacency and orientation metrics of \citet{BOSS} to the ts-MPDAG setting are provided in Appendix~\ref{sec:appendix_metrics}.
% \subsection{Experimental Setup}\label{sec:expset}
Default values for the data generation process and model parameters are reported in Appendix~\ref{appendix:default_params}. For the experiments in this study, the maximum time lag ($lag_{\max}$) is set to the true lag ($\tau_{\max}$). In real applications, where $\tau_{\max}$ is unknown, $lag_{\max}$ can be tuned using domain knowledge and lag-screening checks (e.g., using the Tigramite package). Unless otherwise specified, all experiments use this configuration. To assess performance, we vary the number of samples, number of nodes, graph density and autocorrelation coefficient across experimental settings.
The implementations of the methods and the code used in this work are available at:  \href{https://github.com/irenecastillo/TS-BOSS}{TS-BOSS repository}.

\subsection{Results}
\label{subsec:results}
\begin{figure}[ht]
\centering

\subfigure[Effect of varying sample size ($T$).]{%
    \includegraphics[width=0.48\linewidth]{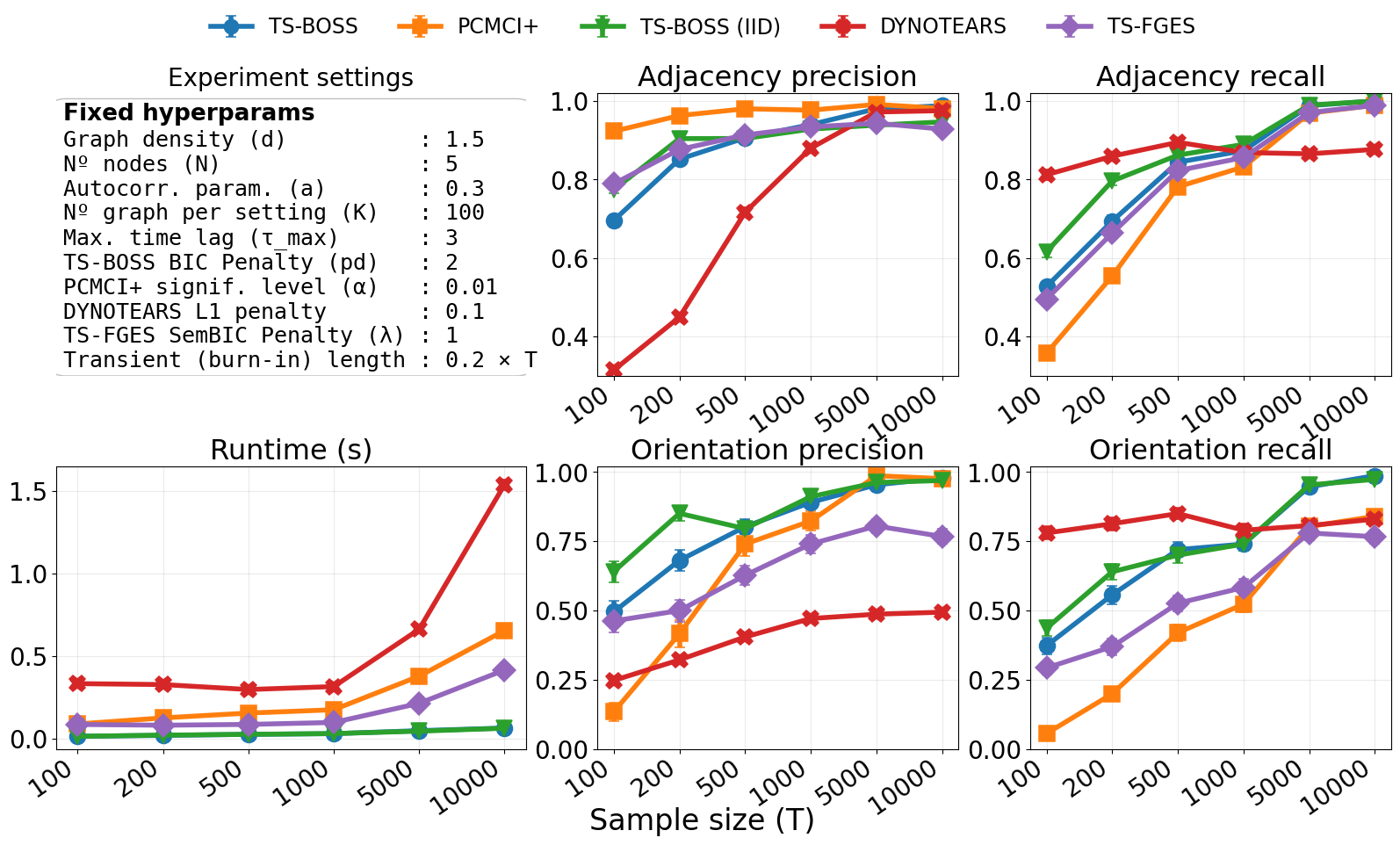}
    \label{fig:exp_samples}}
\subfigure[Effect of varying graph density ($d$).]{%
    \includegraphics[width=0.48\linewidth]{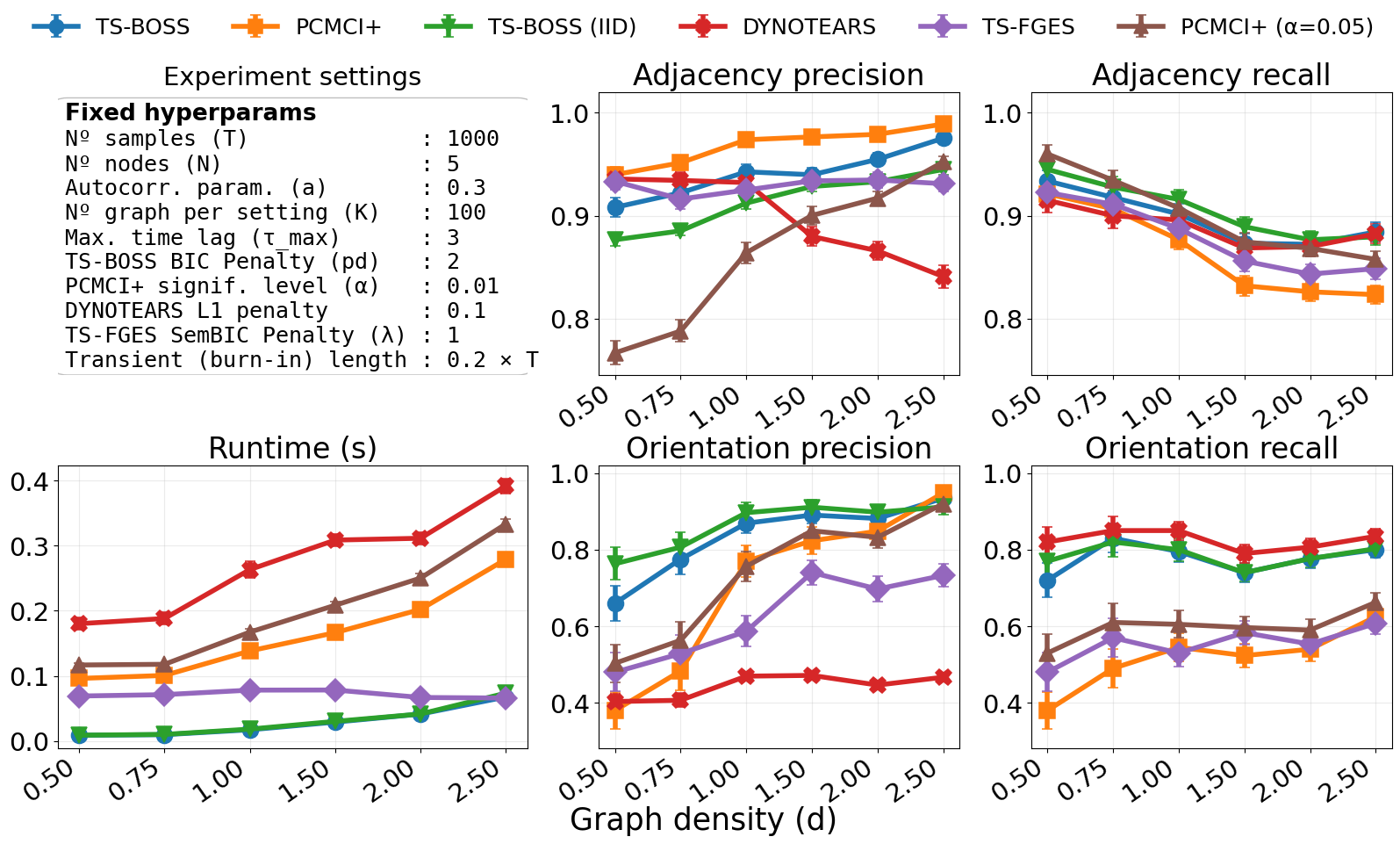}
    \label{fig:exp_density}}
\subfigure[Effect of varying number of nodes ($N$).]{%
    \includegraphics[width=0.48\linewidth]{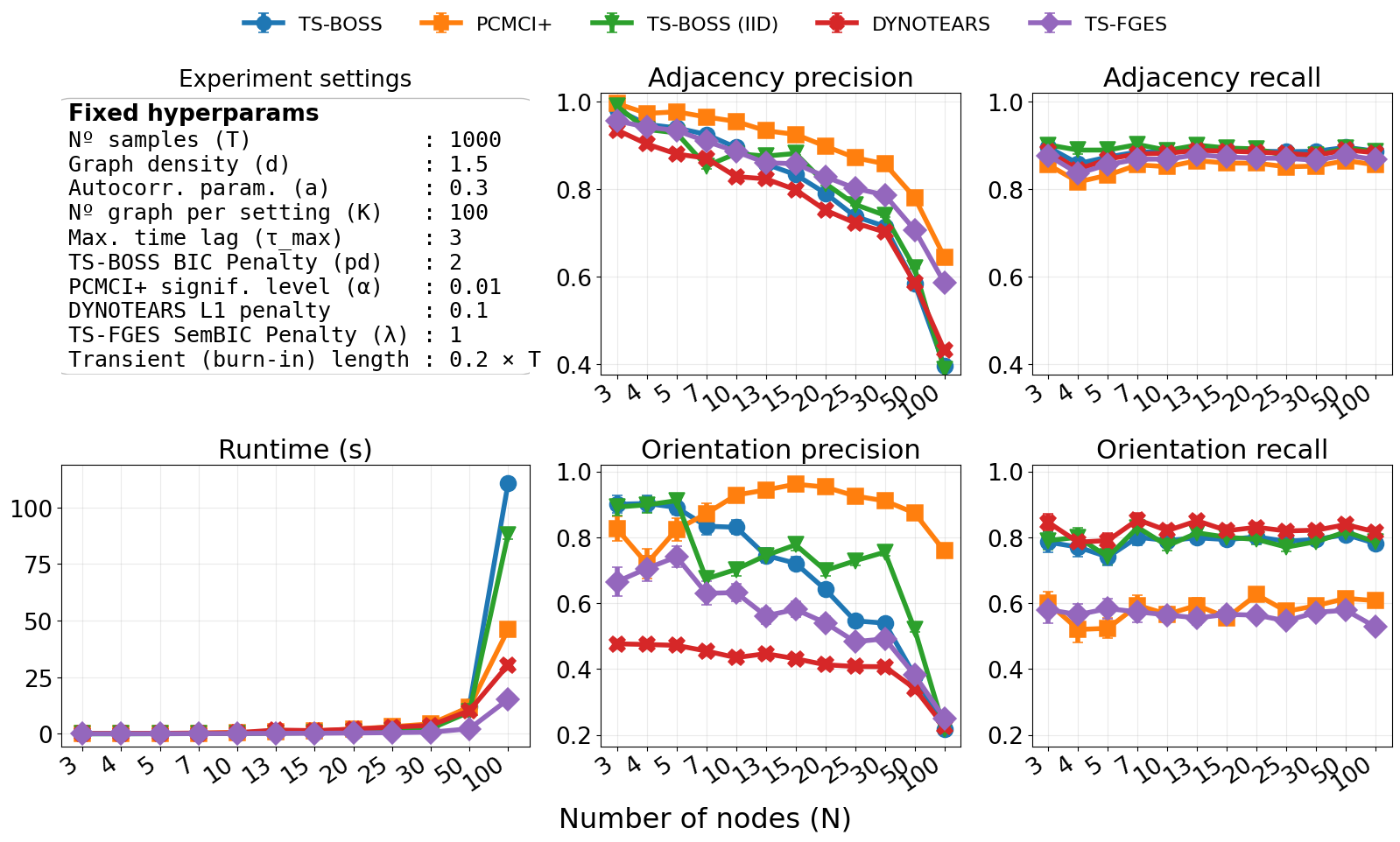}
    \label{fig:exp_nodes}}
\subfigure[Effect of varying autocorrelation ($a$).]{%
    \includegraphics[width=0.48\linewidth]{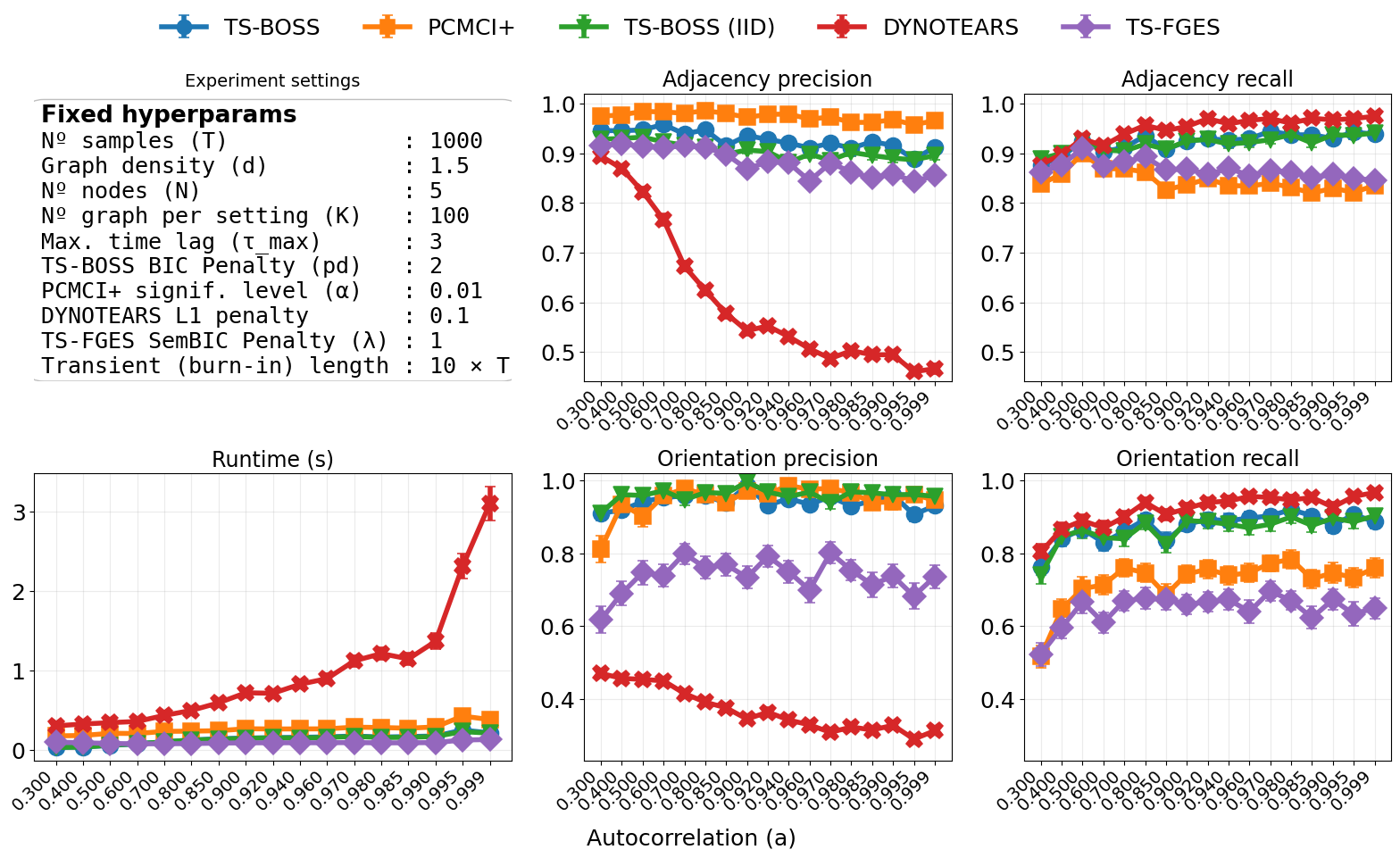}
    \label{fig:exp_auto}}

\subfigure[Effect of varying autocorrelation ($a$): adjacency F1 by edge type.]{%
    \includegraphics[width=0.98\linewidth]{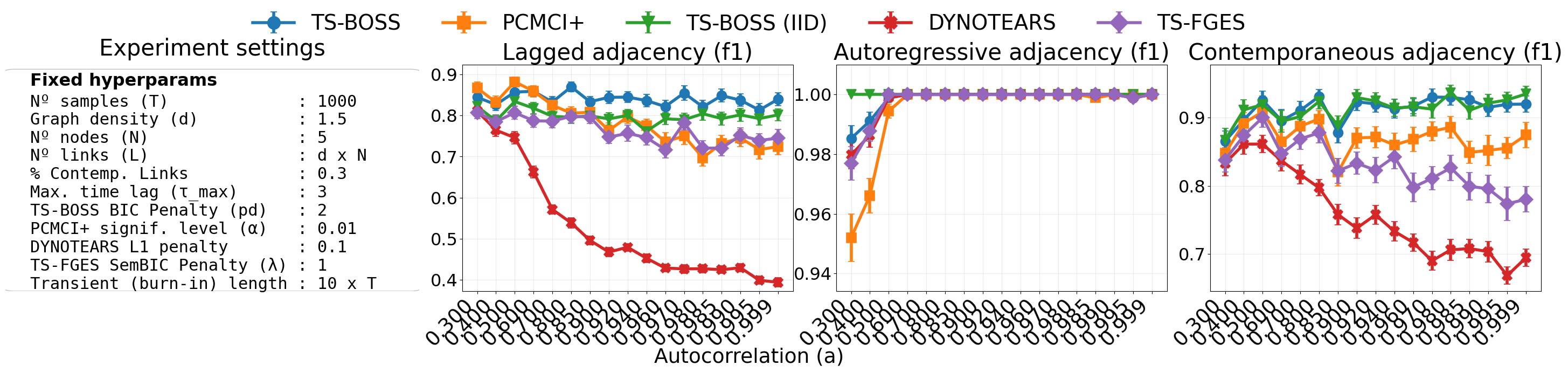}
    \label{fig:exp_auto_adj}
    }

\caption{Results for the methods in \sectionref{sec:models} under the settings defined in Appendix~\ref{appendix:default_params}.}
\label{fig:main_results}
\end{figure}
In Figure~\ref{fig:exp_samples}, increasing the sample size $T$ improves precision and recall for most methods. DYNOTEARS achieves high recall already at small $T$, while PCMCI+ maintains strong adjacency precision across all settings. TS-BOSS and TS-BOSS (i.i.d.) continue improving for larger $T$, reaching competitive or superior adjacency and orientation performance while remaining substantially faster than the competing methods.
Figure~\ref{fig:exp_density} shows the effect of increasing the graph density $d$. As $d$ increases, adjacency recall decreases for all methods. Running PCMCI+ with $\alpha=0.05$ trades precision for recall but still does not outperform TS-BOSS in adjacency recall. DYNOTEARS achieves the highest orientation recall but also the lowest orientation precision. In contrast, TS-BOSS and TS-BOSS (i.i.d.) provide the strongest overall balance between precision and recall while remaining among the fastest methods across most settings.
Figure~\ref{fig:exp_nodes} shows the effect of increasing the number of nodes $N$. Runtime increases for all methods, particularly for TS-BOSS due to the larger permutation search space. PCMCI+ maintains the strongest precision across all settings, while DYNOTEARS and both TS-BOSS variants achieve the highest orientation recall, although DYNOTEARS exhibits the lowest orientation precision. TS-BOSS shows a stronger decline in orientation precision for large $N$.
Finally, Figure~\ref{fig:exp_auto} shows the effect of increasing autocorrelation $a$. PCMCI+ achieves high adjacency precision but decreasing recall, similarly to TS-FGES. DYNOTEARS attains high recall but the weakest precision. In contrast, both TS-BOSS variants remain stable across all values of $a$ and achieve strong adjacency and orientation results. 
The close agreement between TS-BOSS and TS-BOSS~(i.i.d.) across settings also suggests that the BIC score, while not corrected for sliding-window dependence, does not substantially affect performance in the tested regimes. 

To complement the precision and recall results, Figure~\ref{fig:exp_auto_adj} reports adjacency F1-scores by edge type as the autocorrelation $a$ varies. Autoregressive edges are recovered almost perfectly by all methods, while for lagged and contemporaneous edges both TS-BOSS variants remain stable as $a$ grows, unlike PCMCI+, TS-FGES and especially DYNOTEARS. Further F1 results, evaluations against the ground-truth ts-DAG and under high autocorrelation are given in \appendixref{app:further_sims}, and the real-world and memory experiments in \appendixref{app:causal-chamber} and \appendixref{app:memory}.

\section{Conclusion and Outlook}\label{sec:outlook}

We present TS-BOSS, an extension of score-based structure learning method BOSS to multivariate time series. We show that TS-BOSS outperforms standard score- and constraint-based algorithms for time series causal discovery. This suggests that permutation-based score search with efficient caching can remain effective when strong sample dependence degrades the behavior of conditional independence testing.
We presented theoretical guarantees of TS-BOSS under the i.i.d.~window data setting, while illustrating with simulations that non-i.i.d.~data from sliding windows over a single multivariate time series does not hamper TS-BOSS performance. In doing so, we established subgraph minimality of the window causal graph obtained using a modified approach to the one presented in \cite{verma_pearl_networks}, thus paving the way for extending sparse permutation-search based methods for causal discovery to the stationary time series setting.
However, this stationarity assumption may be restrictive in practice and extending TS-BOSS to non-stationary dynamics remains future work.

An important advantage of constraint-based causal discovery is that it can more directly accommodate violations of orientation faithfulness \citep{Ramsey_adjacency_faithfulness}. %, since orientation rules can be adapted to reflect ambiguous or unstable conditional independence patterns.
Score-based methods, by contrast, may be less transparent in how such violations affect the selected graph, and detecting such violations within the score-based framework would be an avenue for future research. %, although they benefit from global optimization criteria that can improve robustness in other regimes.
% On the other hand, permutation-based search appears more amenable to parallelization: scoring candidate parent sets and evaluating local changes can be distributed across nodes in a permutation, which may offer additional scalability benefits for large time series systems.
More broadly, the constraint-based and score-based paradigms impose different methodological bottlenecks in single multivariate time series settings. Constraint-based approaches rely on well-calibrated conditional independence tests under temporal dependence of samples, while score-based approaches require theoretical guarantees —e.g., score consistency— tailored to dynamic Bayesian network structure learning. Each set of requirements is nontrivial in practice and motivates further theoretical work. Our ablation studies provide evidence that increasing sample dependence differentially influences the performance of competing approaches and underline the need for further theoretical guarantees in the single multivariate time series setting.

% \section*{Acknowledgments}

% We thank Jakob Runge for valuable comments, Paul Prasse for support with using the computational resources provided by the University of Potsdam and anonymous reviewers for helpful comments that improved the clarity of our work.

\bibliography{references}

\newpage
\appendix

\section{Proofs}\label{app:proofs}
In this section, we present proofs for theoretical results presented in \sectionref{sec:theory}.
First we recall the following definition from \citep{lam_grasp} that is needed for the proofs.
\begin{definition}[RU projection]\label{def:ru}
    Let $P$ denote a graphoid over a set of vertices $\*V$, and let $\Pi(\*V)$ denote the set of all permutations over the vertices $\*V$. For the $i^{th}$ vertex in $\pi \in \Pi(\*V)$ called $\pi_i \in \sw$, all vertices that precede $\pi_i$ in $\pi$ are denoted by $\pre(i,\pi) = \{\pi_j \ : 1 \leq j < i\}$. We use $X_i$ to denote the random variable corresponding to a vertex $\pi_i$ and $\*X_{\pre(i,\pi)}$ to denote the set of random variables corresponding to the vertex set $\pre(i,\pi)$.  Then, each $\pi \in \Pi(\*V)$ induced a DAG $\cG_\pi$ satisfying the following condition:
    An edge $\pi_j \to \pi_k$ exists in $\cG_\pi$ if and only if $\pi_j \in \pre(k,\pi)$ and $X_j \nind_\cP X_k | \*X_{\pre(k,\pi)\setminus {j}}$.
\end{definition}

\paragraph{Theorem 5  (Permutation-induced window graph minimality)} Let $P_\cW$ denote a graphoid over $\sw$ and $\gw^\pi$ be a window graph induced from an admissible permutation $\pi$. Then $\gw^\pi$ satisfies the window Markov property and is window subgraph minimal. 

\begin{proof}
    For the permutation $\pi$ one can construct the induced DAG following the strategy \emph{RU} (see \definitionref{def:ru}) defined in \citep{lam_grasp} and motivated from \citep{raskutti_sp}, without distinguishing between time indices. It follows that the graph $\cG_{RU} = RU(\pi)$ is such that $\gw^\pi \subseteq \cG_{RU}$, where there may be additional edges in $\cG_{RU}$ due to common confounders of window variables $\*S_\cW$ lying outside the window. Additionally, for each $i$ and the maximal time index $t$ in $\pi$, $\pa(X^i_t)_{\cG_{RU}} = \pa(X^i_t)_{\gw^\pi}$, because the edges incident into the contemporaneous slice have been found using the same strategy in both methods. Further $\nd(X^i_t)_{\cG_{RU}} = \nd(X^i_t)_{\gw^\pi}$ because the permutation $\pi$ was admissible and thus all variables preceding the variables in the contemporaneous slice cannot be their descendants. From Theorem 2 in \citep{verma_pearl_networks}, it follows that $\cG_{RU}$ satisfies $X^i_t \ind \nd(X^i_t)_{\cG_{RU}} \ | \ \pa(X^i_t)_{\cG_{RU}} \Rightarrow X^i_t \ind \nd(X^i_t)_{\gw^\pi} \ | \ \pa(X^i_t)_{\gw^\pi}$. Thus, $\gw^\pi$ satisfies window Markov property. 

    For the subgraph minimality proof, assume the contrary, namely that there exists a stationary subgraph $\hw^\pi \subset \gw^\pi$ that satisfies the window Markov property. Stationarity of the graph implies that $\pa(X^i_t)_{\hw^\pi} \subset \pa(X^i_t)_{\gw^\pi}$. Let $\*D^i_t := \pa(X^i_t)_{\gw^\pi} \setminus \pa(X^i_t)_{\hw^\pi}$. For all $X^j_{t'} \in  \*D^i_t$, $X^j_{t'} \in \nd(X^i_t)_{\hw^\pi}$ because if it were a descendant of $X^i_t$ in $\hw^\pi$, it would also be a descendant of $X^i_t$ in $\gw^\pi$ which would violate acyclicity. From window Markovianity of $\hw^\pi$ we have $\*D^i_t \ind X^i_t \ | \ \pa(X^i_t)_{\hw^\pi}$. By the weak union property of the graphoid axioms we have $X^j_{t'} \ind X^i_t \ | \ (\pa(X^i_t)_{\hw^\pi} \cup (\*D^i_t \setminus X^j_{t'}))$ which is equivalent to $X^j_{t'} \ind X^i_t \ | \ (\pa(X^i_t)_{\gw^\pi}  \setminus X^j_{t'})$. However, this would contradict condition (i) in definition \ref{def:RUprime}. Therefore, $\gw^\pi$ must be window subgraph minimal.  
\end{proof}

\paragraph{Lemma 6}
    %Let $\cB(\cG,\*D)$ be a locally consistent Bayesian scoring criterion for DAG $\cG$ and i.i.d.~window data $\*D$ over the vertices $\sw$.
    Let assumptions \ref{ass:localmarkov}-\ref{ass:asymptotic} be satisfied for data $\*D$ and ts-DAG $\gts$ resulting from a ts-SCM over time series $(\*X_t)_{t\in \integers}$, and a Bayesian scoring criterion $\cB(\cG,\*D)$ over DAGs.
    Then the output graph of phase 1 of TS-BOSS over data $\*D$ with score  $\cB(\cG,\*D)$ satisfies time series local Markov property w.r.t.~$(\*X_t)_{t\in \integers}$ and is subgraph minimal in the large sample limit.
    % Then thw phase 1 of TS-BOSS returns the MEC of the true window causal graph $\gw$ in the large sample limit.

\begin{proof}
Suppose not, then the output ts-DAG $\widehat{\gts}$ of phase 1 of TS-BOSS is such that there exists an $X^j_{s}\in \nd(X^i_t)_{\widehat{\gts}}$ such that $X^i_t \nind X^j_{s} \ | \ \pa(X^i_t)_{\widehat{\gts}} $. From \assumptionref{ass:localmarkov}, \assumptionref{ass:maxlag} and \assumptionref{ass:station}, it follows that there exists a variable that is a non-descendant in the window graph and violates the window Markov property. W.l.o.g.~let this variable be $X^j_{s}$. Consider the last permutation in phase 1 before the permutation search terminates because no score improvement can be found. Local score consistency of $\cB(\cG,\*D)$, i.e.~\assumptionref{ass:asymptotic},  implies that the ts-DAG that results from adding edge $X^j_{s} \to X^i_t$ when computing the permutation-induced graph, has a higher score than for a ts-DAG without this edge, therefore the first phase cannot terminate at $\widehat{\gts}$. 
% \begin{itemize}
%     \item[-] RU' yields a window causal graph $\widehat{G}$ from a permutation $\pi$. Let $\widehat{G}_{A}$ be the projected ADMG from the infinite ts-DAG corresponding to $\widehat{G}$. Consider graphoid $P_\cW$ that is the marginal distribution over $\sw$. 
%     Then, every contemp-contemp and contemp-lagged d-separation in the finite
% \end{itemize}
%     Idea: Prove TS-BOSS phase 1 soundness by defining modified permutation induced graph and using Verma Pearl 
%     (Modified permutation yields window graph. Window graph yields entire ts-DAG. From ts-DAG, project to any finite subset and get projected ADMG. Correspondingly consider the marginal distribution over the projected nodes which is a graphoid. Extend Verma-Pearl, i.e. Markovianity and SGS-minimality, for the projected ADMG and projected distribution. Assuming ts-Markov global property will be needed.)
\end{proof}

\paragraph{Lemma 7}
    Let assumptions \ref{ass:localmarkov}-\ref{ass:station} be satisfied for data $\*D$ and ts-DAG $\gts$ resulting from a ts-SCM over time series $(\*X_t)_{t\in \integers}$. 
    Let $\widehat{\cG}_{ts}$ be a ts-DAG that satisfies the time series local Markov property w.r.t.~$(\*X_t)_{t\in \integers}$. %corresponding to a multivariate time series that satisfies \cref{ass:localmarkov}.
    Then TS-BES with input graph $\widehat{\cG}_{ts}$, and a Bayesian scoring criterion $\cB(\cG,\*D)$ and time series data $\*D$ that satisfies \assumptionref{ass:asymptotic}, returns the ts-MPDAG corresponding to the true ts-DAG $\gts$ in the large sample limit.

\begin{proof}
    %Idea: Prove TS-BES by focusing on induced graph on contemp nodes and lagged to contemp nodes.
    TS-BES takes as input the graph $\widehat{\cG}_{ts}$. Let $\widehat{\gw}$ be the corresponding window graph and $\gw$ be the true window graph. Suppose $\gw \subset \widehat{\gw}$. For both these graphs consider only the edges incident into the contemporaneous slice because TS-BES attempts deletion of these edges only by \assumptionref{ass:maxlag} and \assumptionref{ass:station}. Following the proof of soundness of BES in \citep{Chickering_ges}, we note that Theorem 4 therein, namely the Meek conjecture, implies that the the true $\gw$ can be obtained from the estimated $\widehat{\gw}$ with a sequence of edge deletions and covered edge reversals. Consider the last graph $\cH$ in the sequence from $\gw$ to $\widehat{\gw}$ which has fewer edges than $\widehat{\gw}$. By \assumptionref{ass:asymptotic}, $\cH$ has a higher score than $\widehat{\gw}$, thus TS-BES cannot terminate at $\widehat{\gw}$.
\end{proof}

\section{Construction of the ts-MPDAG from a ts-DAG} 
\label{app:tscpdag}

Given a ts-DAG, the corresponding ts-MPDAG is constructed as follows:

\begin{enumerate}
    \item All edges are first rendered unoriented.
    
    \item The orientation of all lagged edges is then fixed, since temporal order uniquely determines their direction.
    
    \item All colliders are oriented, including both contemporaneous colliders and mixed-time colliders. A mixed-time collider refers to a v-structure of the form
    \[
    X_{t-\tau}^i \rightarrow Y_t^j \leftarrow Z_t^k,
    \]
    where one parent is lagged and the other contemporaneous.
    
    \item Meek’s rule R1 \citep{meek2013causalinferencecausalexplanation} is applied to unshielded triples involving both lagged and contemporaneous edges, preventing the introduction of additional mixed-time colliders.
    
    \item Finally, Meek’s rules R1--R3 \citep{meek2013causalinferencecausalexplanation} are applied to the contemporaneous subgraph until no further orientations are possible.
    Note that R2 and R3 only need to be applied to the contemporaneous subgraph, because no subgraph satisfying the precondition for R2 and R3 contains mixed-time triples.

\end{enumerate}
\section{Evaluation Metrics}
\label{sec:appendix_metrics}

The evaluation protocol introduced in \citep{BOSS} computes performance metrics by comparing the estimated graphs produced by the evaluated methods to the corresponding ground-truth DAGs. In this work, since the estimated graphs are ts-MPDAGs, the ground-truth ts-DAG is first converted into the corresponding ts-MPDAG (Appendix~\ref{app:tscpdag}), and all metrics are computed at the ts-MPDAG level. The evaluation protocol described in this section extends the adjacency and orientation metrics of \citep{BOSS} to account for undirected edges in ts-MPDAGs.

As mentioned in \sectionref{sec:evalmetrics}, we report adjacency precision and recall, orientation precision and recall, and runtime. Precision and recall are computed as

\[
\text{Precision} = \frac{TP}{TP+FP},
\qquad
\text{Recall} = \frac{TP}{TP+FN}.
\]

The definitions of true positives (TP), false positives (FP), false negatives (FN) and true negatives (TN) differ for adjacency and orientation evaluation and are described in the following subsections.

% To enable evaluation at the ts-CPDAG level, the adjacency and orientation metrics introduced in \citep{BOSS} are extended to account for undirected edges as follows.

\subsection{Adjacency Evaluation}

Adjacency precision and recall are computed by evaluating whether an edge exists between two nodes independently of its orientation. Since adjacency evaluation depends only on edge existence, the presence of undirected edges in ts-MPDAGs does not modify the DAG-based evaluation protocol introduced in \citep{BOSS}. The corresponding adjacency evaluation outcomes are summarized in \tableref{tab:metrics}.

\subsection{Orientation Evaluation}

For orientation evaluation, we adopt the arrowhead-based criterion introduced in \citep{BOSS}. For each ordered pair $(X_{t-\tau}^{i}, X_{t-\tau}^{j})$ with $\tau = 0$, orientation is evaluated locally by checking whether the edge between $X_t^{i}$ and $X_t^{j}$ contains an arrowhead into node $X_t^{i}$. Orientation evaluation is therefore restricted to contemporaneous edges, since lagged edges are uniquely oriented by temporal order. The arrowhead check is performed separately for each node of every edge.

Let $\mathbb{I}_{\rightarrow i}(j)$ denote the indicator function defined as
\[
\mathbb{I}_{\rightarrow i}(j) =
\begin{cases}
1 & \text{if the edge between } X_t^{i} \text{ and } X_t^{j} \text{ contains an arrowhead into } X_t^{i}, \\
0 & \text{otherwise.}
\end{cases}
\]

The orientation evaluation outcomes and the corresponding indicator values for the ground-truth ts-MPDAG and the estimated ts-MPDAG are summarized in Table~\ref{tab:metrics}, where $\mathbb{I}_{\rightarrow i}^{\text{true}}(j)$ denotes the indicator value in the ground-truth graph and $\mathbb{I}_{\rightarrow i}^{\text{est}}(j)$ the corresponding value in the estimated graph.

If the ground-truth graph contains an undirected edge $X_t^{i} \circ\!\!-\!\!\circ X_t^{j}$, then $\mathbb{I}_{\rightarrow i}^{\text{true}}(j)=0$, since no arrowhead into $X_t^{i}$ is present. Consequently, predicting no arrowhead into $X_t^{i}$ is counted as correct.
% \begin{itemize}
%     \item TP: $\mathbb{I}_{\rightarrow i}^{\text{true}}(j)=1$ and $\mathbb{I}_{\rightarrow i}^{\text{est}}(j)=1$,
    
%     \item FP: $\mathbb{I}_{\rightarrow i}^{\text{true}}(j)=0$ and $\mathbb{I}_{\rightarrow i}^{\text{est}}(j)=1$,
    
%     \item FN: $\mathbb{I}_{\rightarrow i}^{\text{true}}(j)=1$ and $\mathbb{I}_{\rightarrow i}^{\text{est}}(j)=0$,
    
%     \item TN: $\mathbb{I}_{\rightarrow i}^{\text{true}}(j)=0$ and $\mathbb{I}_{\rightarrow i}^{\text{est}}(j)=0$.
% \end{itemize}
\begin{table}[htbp]
\centering
\small
\setlength{\tabcolsep}{4pt}
\begin{tabular}{|c c || c || c c c |}
\hline
True edge & Estimated edge 
& Adjacency 
& $\mathbb{I}_{\rightarrow i}^{\text{true}}(j)$ 
& $\mathbb{I}_{\rightarrow i}^{\text{est}}(j)$ 
& Orientation \\
\hline

$X_t^{i} \leftarrow X_t^{j}$ & $X_t^{i} \leftarrow X_t^{j}$ 
& TP & 1 & 1 & TP \\

$X_t^{i} \leftarrow X_t^{j}$ & $X_t^{i} \rightarrow X_t^{j}$ 
& TP & 1 & 0 & FN \\

$X_t^{i} \leftarrow X_t^{j}$ & $X_t^{i} \circ\!\!-\!\!\circ X_t^{j}$ 
& TP & 1 & 0 & FN \\

$X_t^{i} \leftarrow X_t^{j}$ & $X_t^{i} \dots X_t^{j}$ 
& FN & 1 & 0 & FN \\
\hline

$X_t^{i} \rightarrow X_t^{j}$ & $X_t^{i} \leftarrow X_t^{j}$ 
& TP & 0 & 1 & FP \\

$X_t^{i} \rightarrow X_t^{j}$ & $X_t^{i} \rightarrow X_t^{j}$ 
& TP & 0 & 0 & TN \\

$X_t^{i} \rightarrow X_t^{j}$ & $X_t^{i} \circ\!\!-\!\!\circ X_t^{j}$ 
& TP & 0 & 0 & TN \\

$X_t^{i} \rightarrow X_t^{j}$ & $X_t^{i} \dots X_t^{j}$ 
& FN & 0 & 0 & TN \\
\hline

$X_t^{i} \circ\!\!-\!\!\circ X_t^{j}$ & $X_t^{i} \leftarrow X_t^{j}$ 
& TP & 0 & 1 & FP \\

$X_t^{i} \circ\!\!-\!\!\circ X_t^{j}$ & $X_t^{i} \rightarrow X_t^{j}$ 
& TP & 0 & 0 & TN \\

$X_t^{i} \circ\!\!-\!\!\circ X_t^{j}$ & $X_t^{i} \circ\!\!-\!\!\circ X_t^{j}$ 
& TP & 0 & 0 & TN \\

$X_t^{i} \circ\!\!-\!\!\circ X_t^{j}$ & $X_t^{i} \dots X_t^{j}$ 
& FN & 0 & 0 & TN \\
\hline

$X_t^{i} \dots X_t^{j}$ & $X_t^{i} \leftarrow X_t^{j}$ 
& FP & 0 & 1 & FP \\

$X_t^{i} \dots X_t^{j}$ & $X_t^{i} \rightarrow X_t^{j}$ 
& FP & 0 & 1 & FP \\

$X_t^{i} \dots X_t^{j}$ & $X_t^{i} \circ\!\!-\!\!\circ X_t^{j}$ 
& FP & 0 & 0 & -- \\

$X_t^{i} \dots X_t^{j}$ & $X_t^{i} \dots X_t^{j}$ 
& TN & 0 & 0 & -- \\
\hline
\end{tabular}

\caption{Adjacency and orientation evaluation outcomes for ts-MPDAG comparison.}
\label{tab:metrics}
\end{table}

\section{Default Experimental Parameters}
\label{appendix:default_params}

Table~\ref{tab:default_params} reports the default parameters used for the data generation process defined in \sectionref{sec:datagen} and for the models compared in this work (see \sectionref{sec:models}), unless otherwise specified.

\begin{table}[htbp]\centering\small\begin{tabular}{l l}
\hline Parameter & Default value \\\hline\multicolumn{2}{l}{\textbf{Data generation}} \\Number of variables (nodes) $N$ & 5 \\Sample size $T$ & 1000 \\Number of graphs per setting $K$ & 100 \\Coefficient distribution & discrete in $[-0.5,0.5]$ \\Graph density $d$ & 1.5 \\Number of links $L$ & $\lfloor d \times N \rfloor$ \\$\%$ of contemporaneous links & 0.3 \\Autocorrelation parameter $a$ & 0.3 \\Autocorrelation sampling & $U([\max(0.1,a-0.3),\,a])$ \\Coupling function family & linear \\Maximum true time lag $\tau_{\max}$ & 3 \\Transient (burn-in) length & $0.2 \times T$ \\Nonstationary realizations & discarded \\\multicolumn{2}{l}{\textbf{TS-BOSS}} \\Maximum time lag $lag_{\max}$ & $\tau_{\max}$ \\Score function & BIC \\BIC penalty weight $w$ & $\frac{pd}{4} \cdot \frac{\log(T)}{T}$ \\BIC penalty discount $pd$ & 2 \\\multicolumn{2}{l}{\textbf{PCMCI+}} \\Maximum time lag $lag_{\max}$ & $\tau_{\max}$ \\Significance level & $\alpha = 0.01$ \\Contemporaneous rule & Majority \\Conditional test & ParCorr ($significance=analytic$) \\Minimum lag $\tau_{\min}$ & 0 \\\multicolumn{2}{l}{\textbf{TS-FGES}} \\Maximum time lag $lag_{\max}$ & $\tau_{\max}$ \\Score function & SemBicScore \\Penalty discount & 1.0 \\Faithfulness assumption  & True \\Symmetric first step & False \\Replicating & True \\Maximum degree& -1 \\Number of threads & 1 \\\multicolumn{2}{l}{\textbf{DYNOTEARS}} \\Maximum time lag $lag_{\max}$ & $\tau_{\max}$ \\lambda\_w & 0.1 \\lambda\_a & 0.1 \\Weight threshold & $10^{-4}$ \\Maximum iterations & 100 \\Tolerance $h_{\text{tol}}$ & $10^{-8}$ \\\hline\end{tabular}\caption{Default parameter settings used throughout the simulation study in  \sectionref{sec:experiments}.}\label{tab:default_params}\end{table}

\section{Further Simulations}\label{app:further_sims}

\subsection{Results Compared with the Ground-Truth TS-DAG}
\label{sec:results_tsdags}

To mirror the evaluation protocol used in \citep{BOSS}, we additionally report results for the experiments presented in \sectionref{sec:experiments}, comparing the ground-truth ts-DAG with the ts-MPDAG returned by each method. The results are shown in \figureref{fig:results_dag}.

\begin{figure}[htbp]
\centering
\subfigure[Effect of varying sample size ($T$).]{%
    \includegraphics[width=0.48\linewidth]{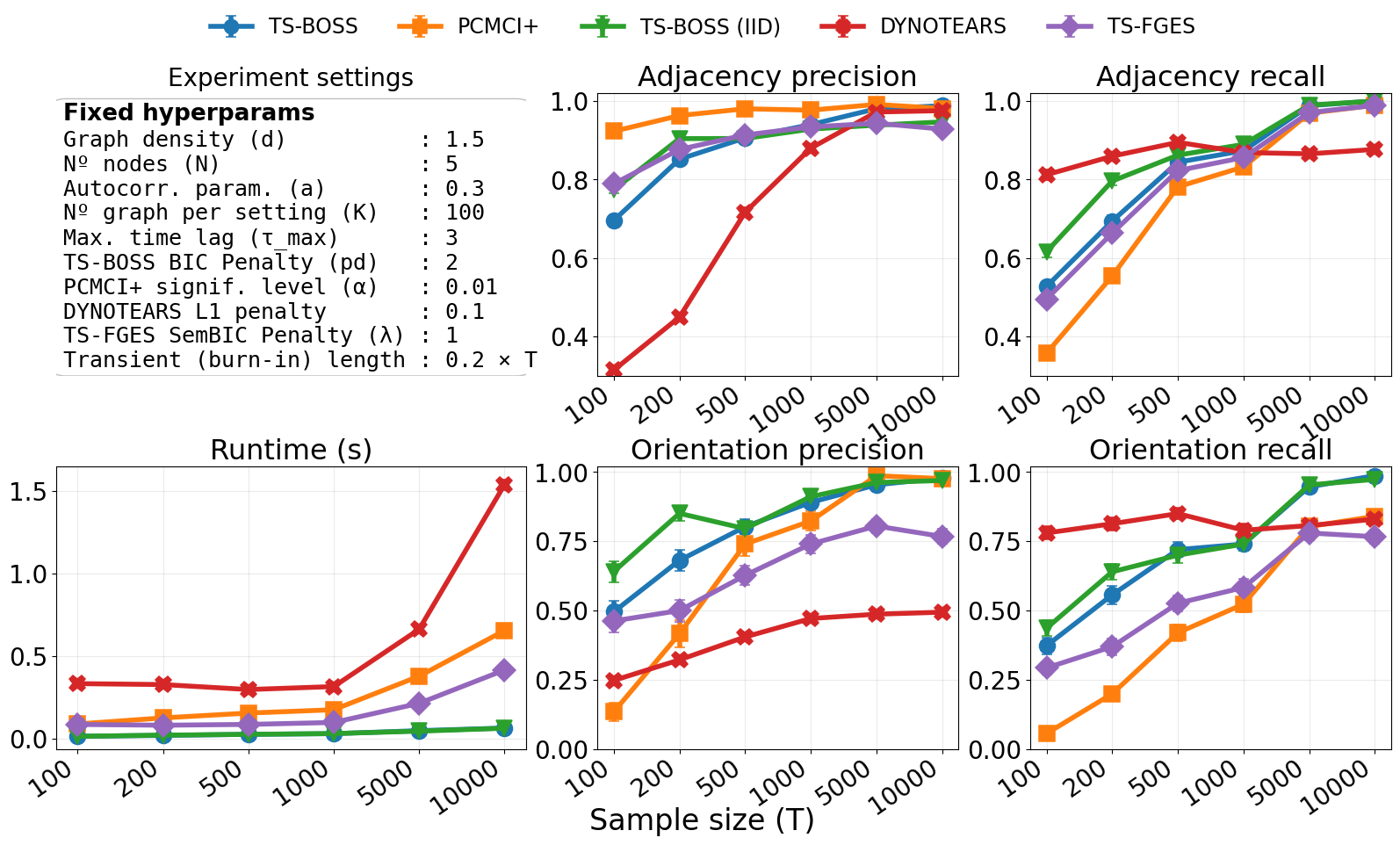}
    \label{fig:exp_samples_dag}}
\hfill
\subfigure[Effect of varying graph density ($d$).]{%
    \includegraphics[width=0.48\linewidth]{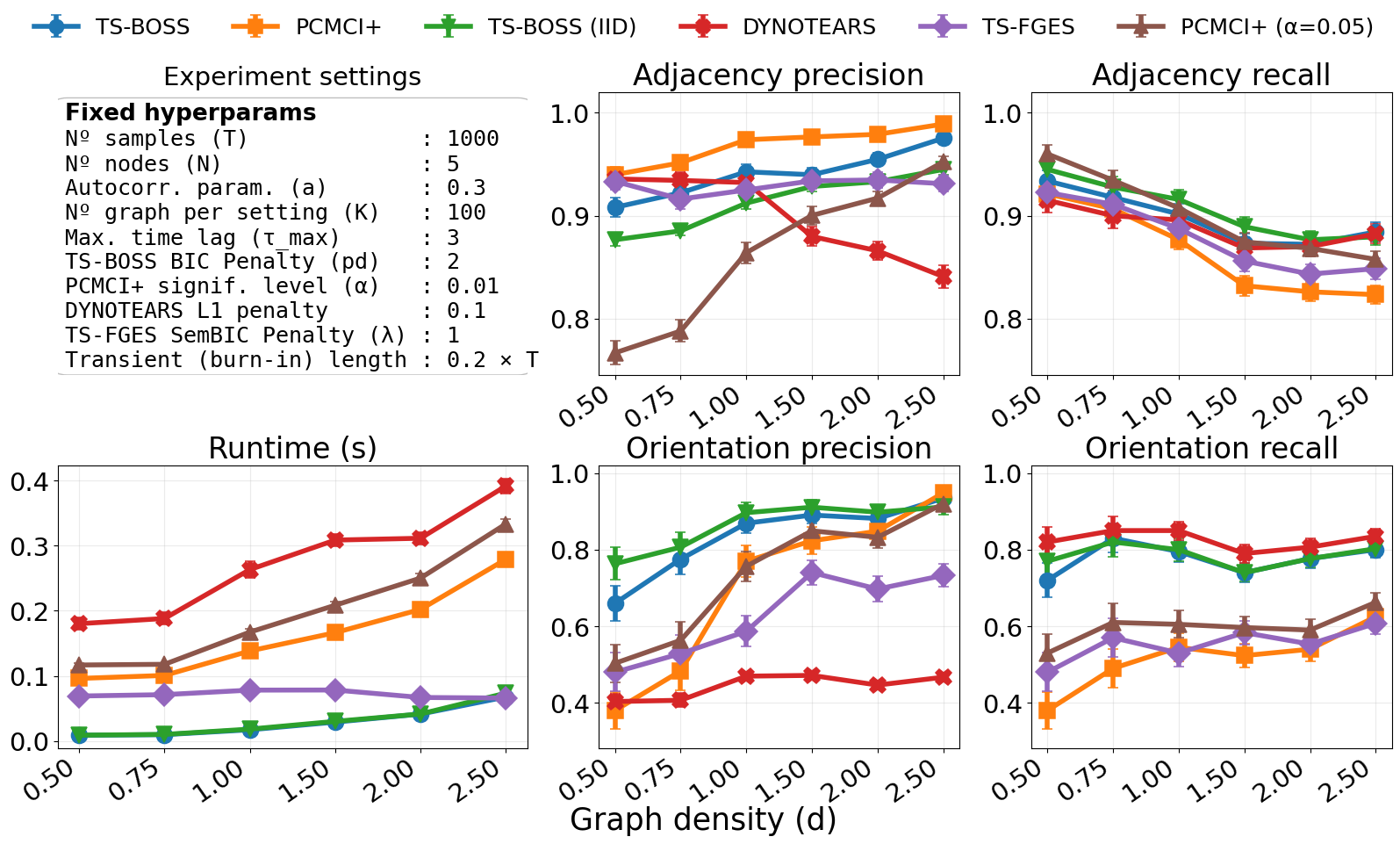}
    \label{fig:exp_density_dag}}
\medskip
\subfigure[Effect of varying number of nodes ($N$).]{%
    \includegraphics[width=0.48\linewidth]{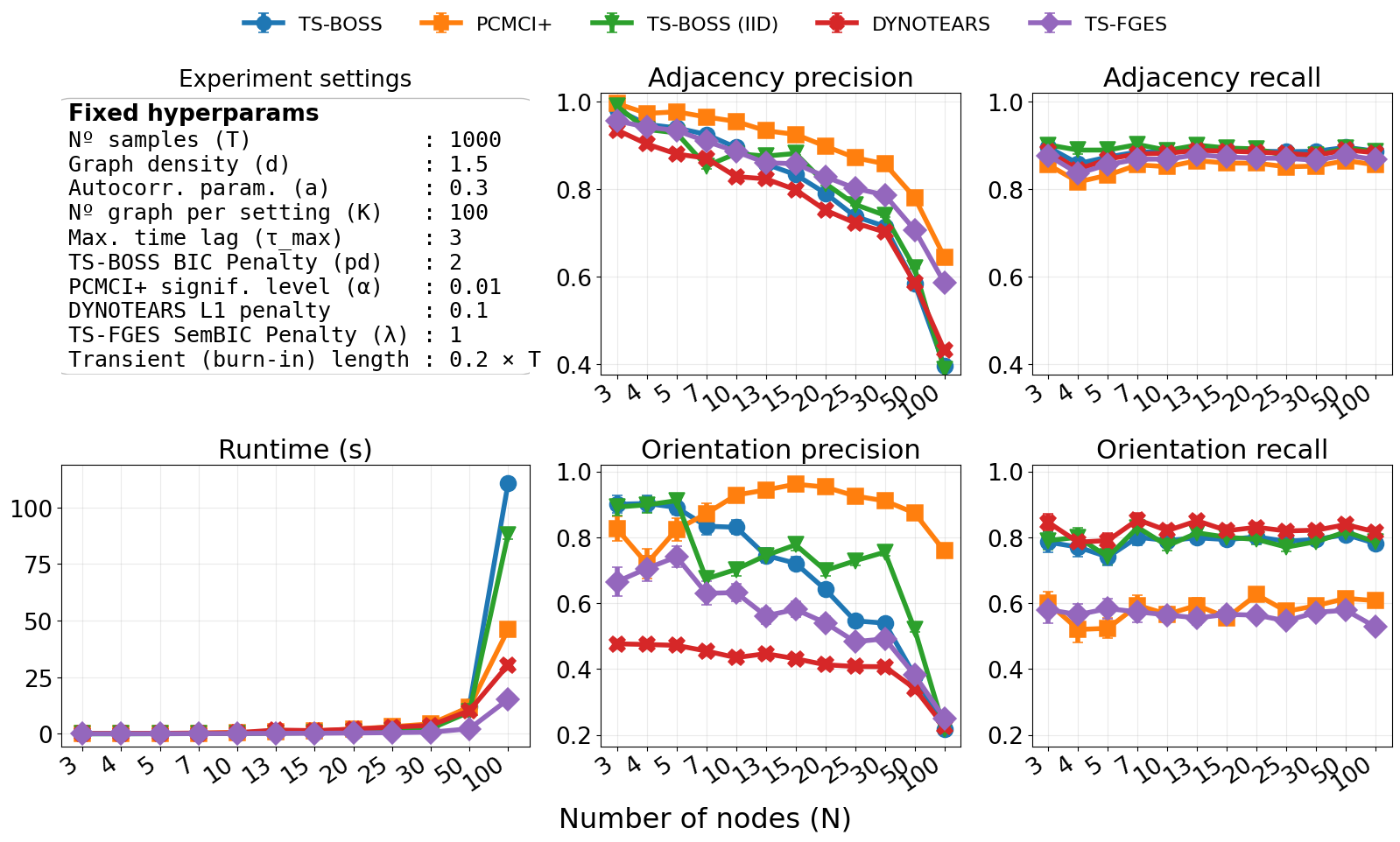}
    \label{fig:exp_nodes_dag}}
\hfill
\subfigure[Effect of varying autocorrelation ($a$).]{%
    \includegraphics[width=0.48\linewidth]{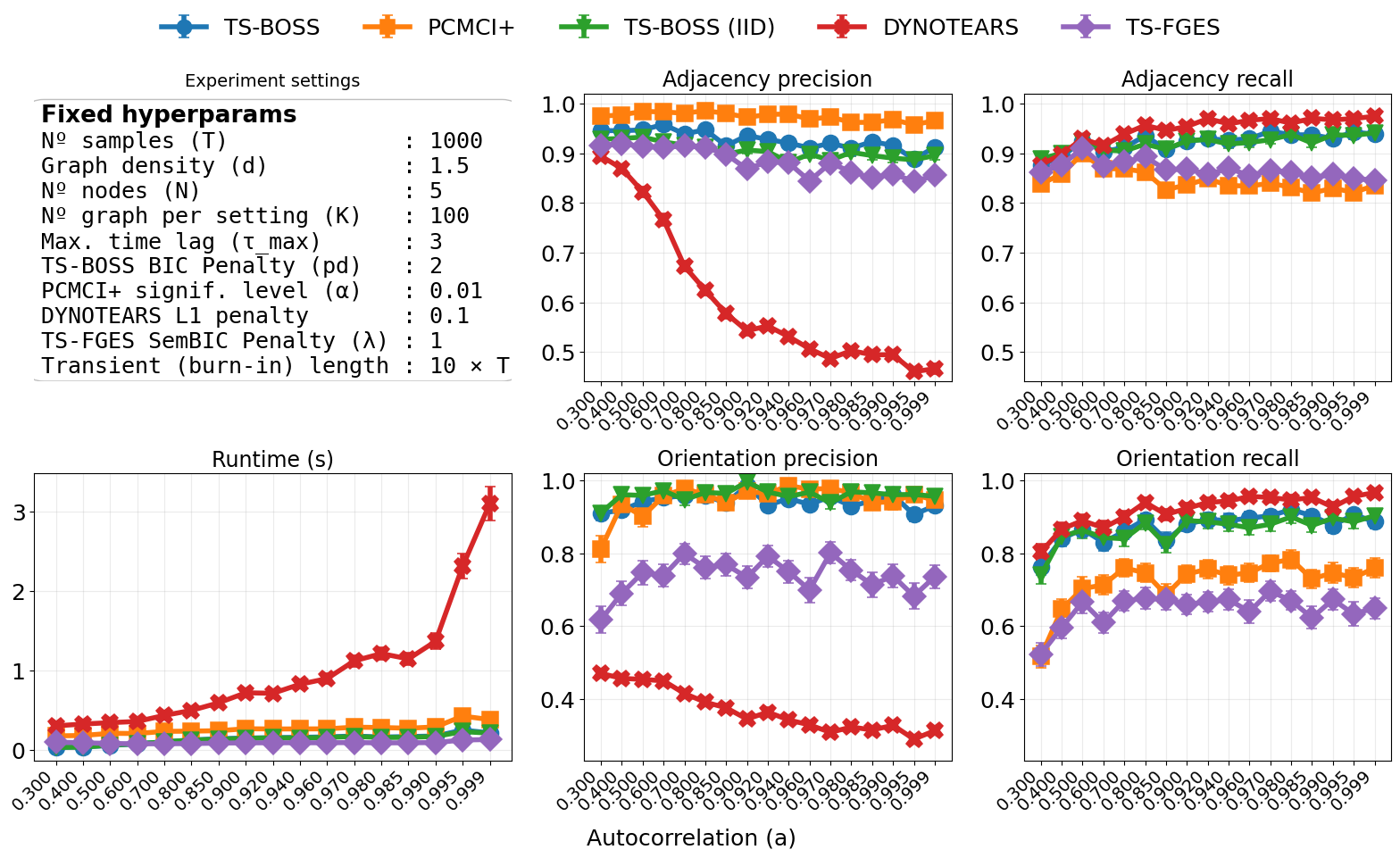}
    \label{fig:exp_auto_dag}}
\caption{Results comparing the ground-truth ts-DAG to the ts-MPDAG returned by each method under the experimental settings defined in \sectionref{sec:experiments}.}
\label{fig:results_dag}
\end{figure}

\clearpage
\subsection{Results under High Autocorrelation}

We further repeat the experiments described in \sectionref{sec:experiments} under a fixed autocorrelation parameter of $a = 0.95$. The results are shown in Figure~\ref{fig:main_results_ac095}.

\begin{figure}[htbp]
\centering

\subfigure[Effect of varying sample size ($T$).]{%
    \includegraphics[width=0.48\linewidth]{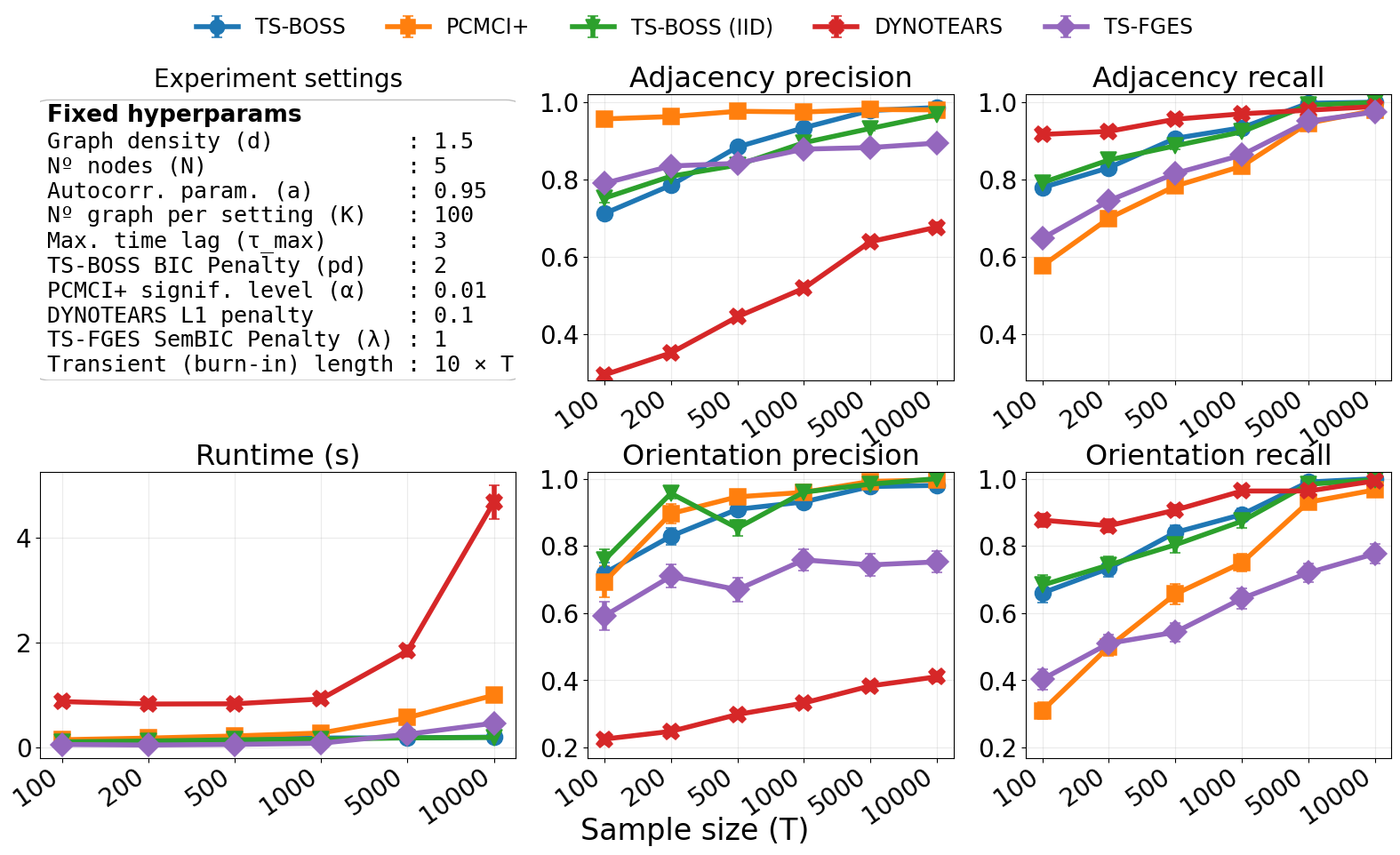}
    \label{fig:exp_samples_ac095}}
\hfill
\subfigure[Effect of varying graph density ($d$).]{%
    \includegraphics[width=0.48\linewidth]{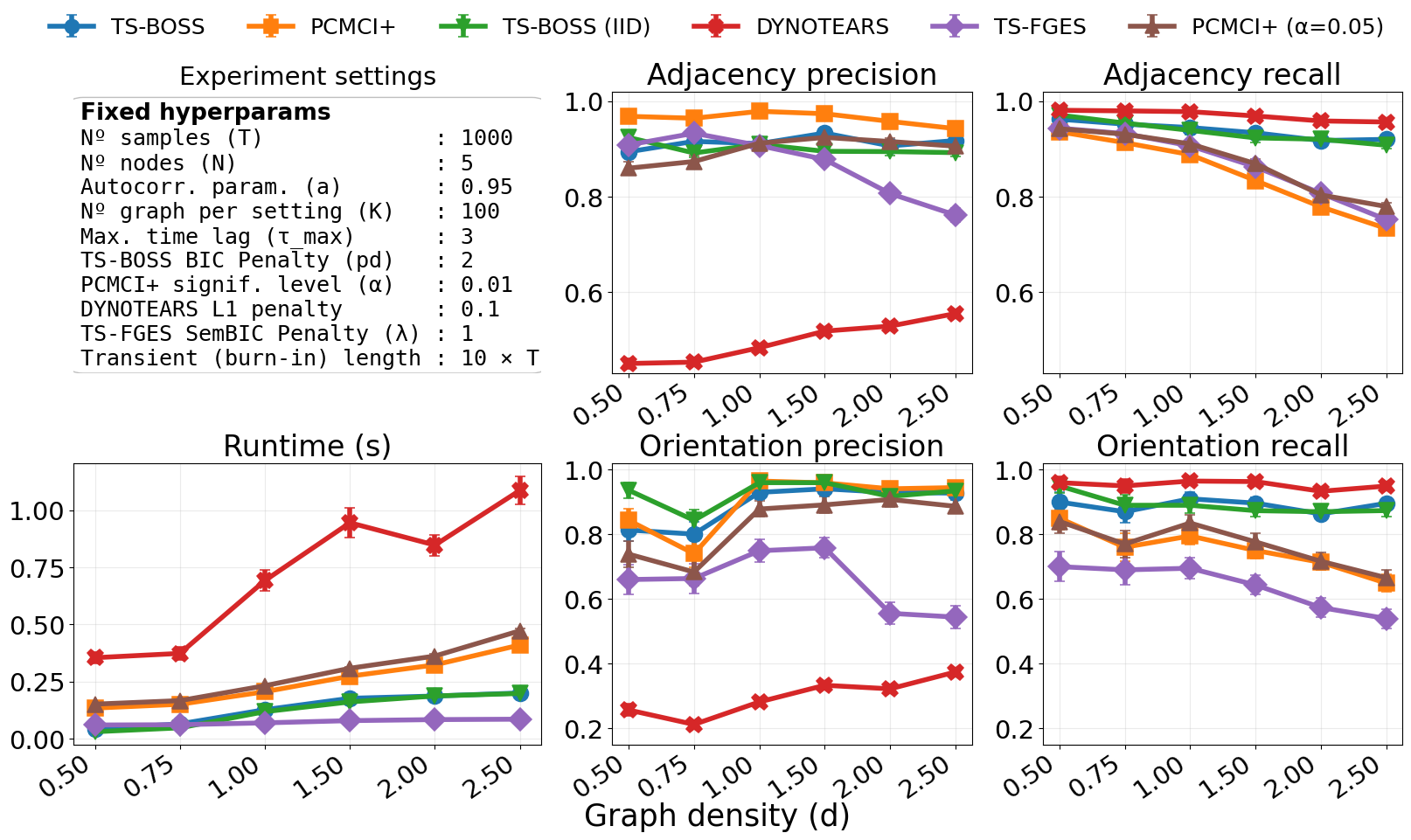}
    \label{fig:exp_density_ac095}}

\medskip

\subfigure[Effect of varying number of nodes ($N$).]{%
    \includegraphics[width=0.48\linewidth]{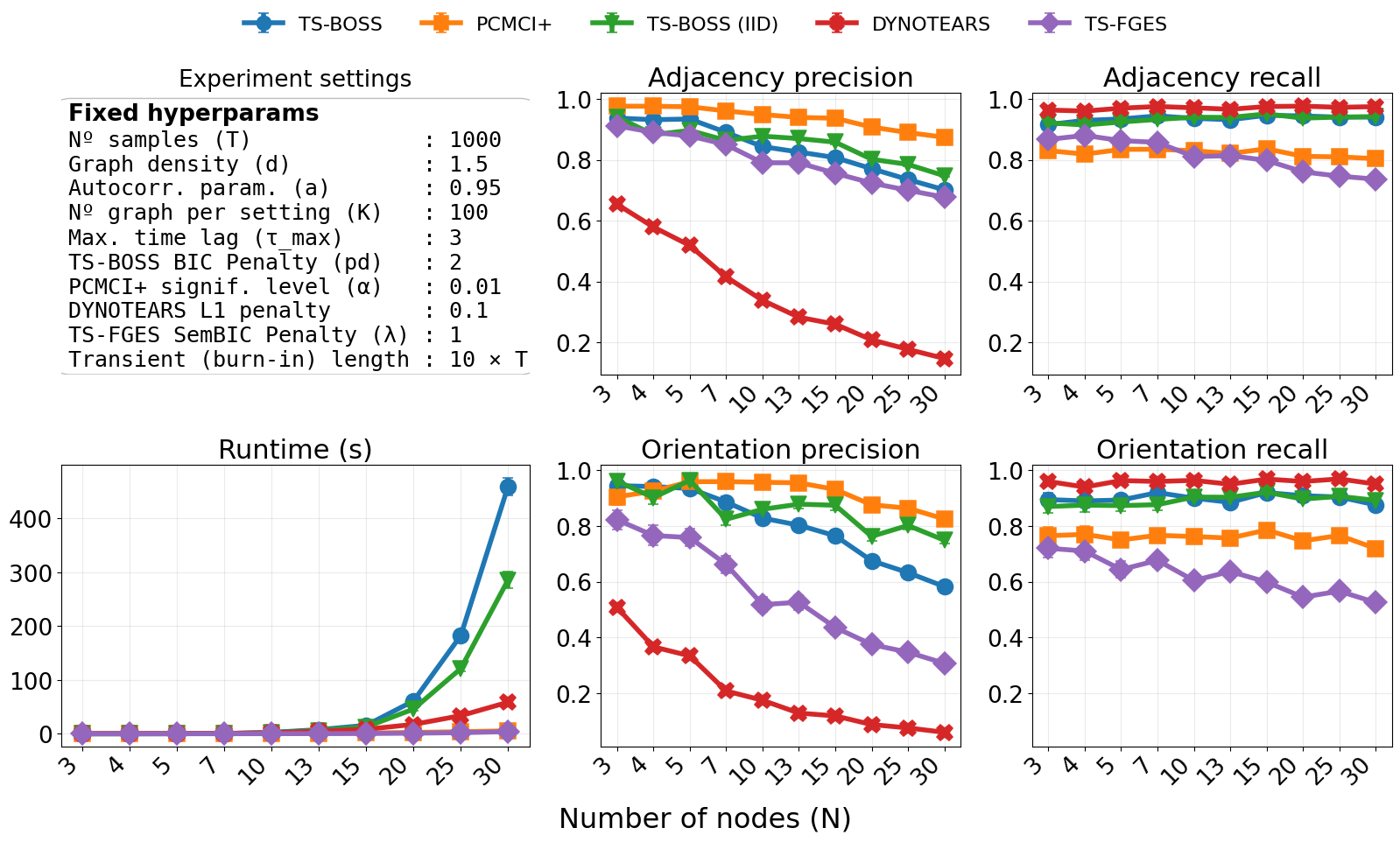}
    \label{fig:exp_nodes_ac095}}

\caption{Results under the experimental settings defined in \sectionref{sec:experiments} with fixed autocorrelation $a = 0.95$.}
\label{fig:main_results_ac095}

\end{figure}

\clearpage
\subsection{Adjacency Results by Edge Type}

For the experimental settings defined in \sectionref{sec:experiments}, we further report adjacency performance separately for contemporaneous, lagged and autoregressive edges.  The results are shown in \figureref{fig:main_results_adjacency_f1}.
\begin{figure}[htbp]
\centering

\subfigure[Effect of varying sample size ($T$).]{%
    \includegraphics[width=0.98\linewidth]{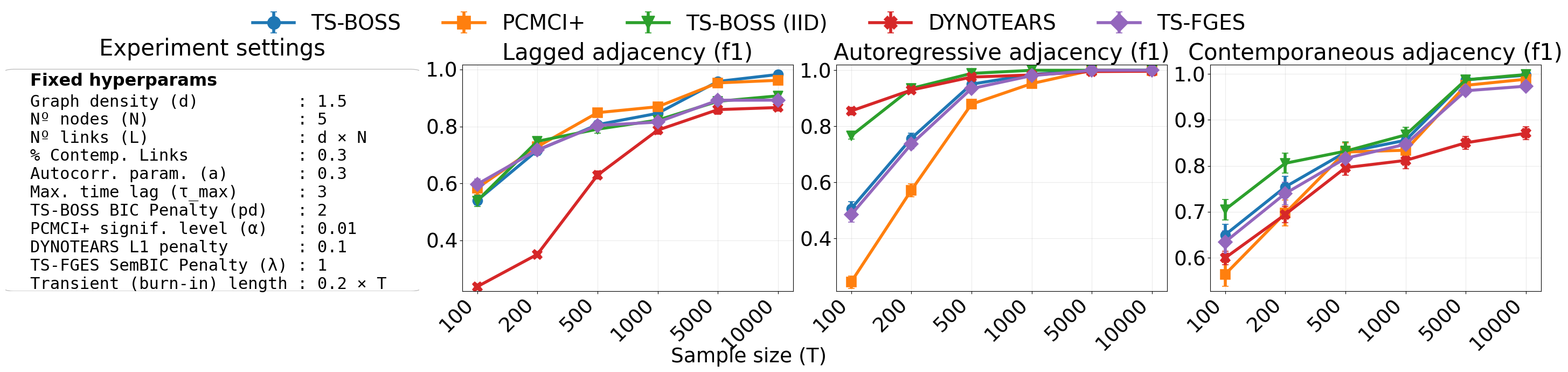}
    \label{fig:exp_samples_adj}}

\subfigure[Effect of varying graph density ($d$).]{%
    \includegraphics[width=0.98\linewidth]{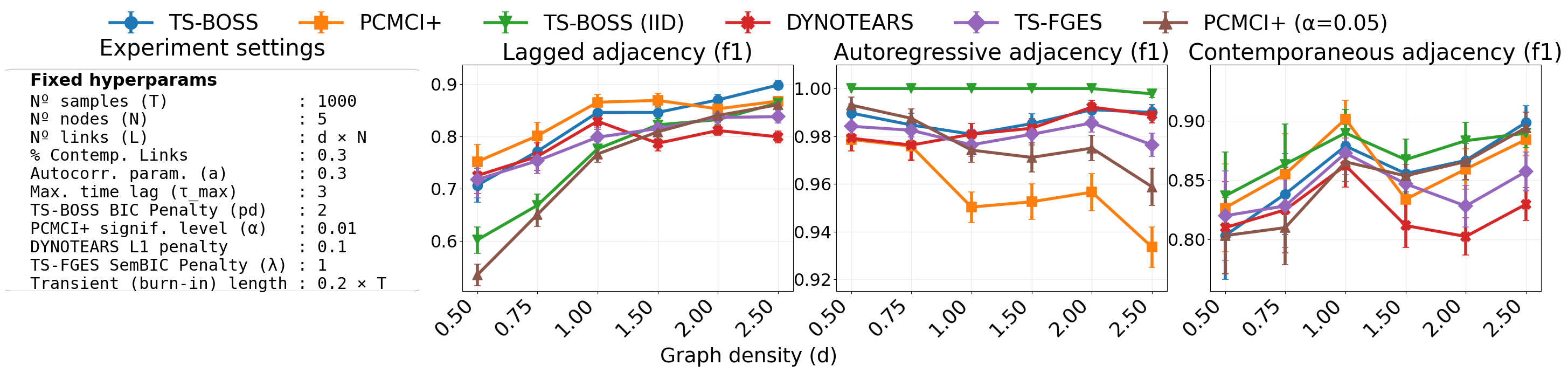}
    \label{fig:exp_density_adj}}

\subfigure[Effect of varying number of nodes ($N$).]{%
    \includegraphics[width=0.98\linewidth]{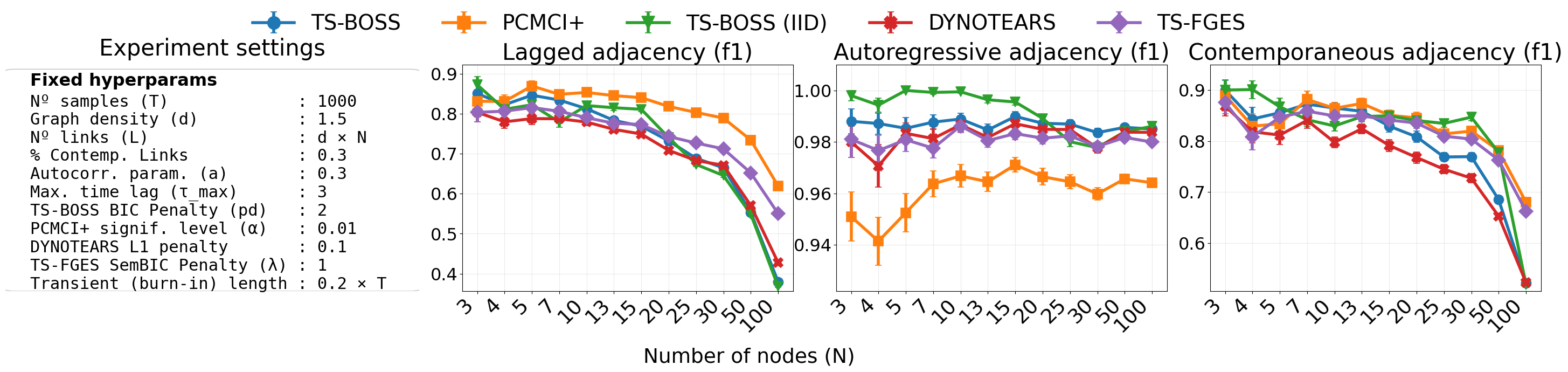}
    \label{fig:exp_nodes_adj}}

% \subfigure[Effect of varying autocorrelation ($a$).]{%
%     \includegraphics[width=0.98\linewidth]{figures/experiment_autocorrelation_adjacency_default_v2.png}
%     \label{fig:exp_auto_adj}
%     }

\caption{Adjacency F1-scores reported separately for lagged, autoregressive and contemporaneous  edges under the experimental settings defined in \sectionref{sec:experiments}.}

\label{fig:main_results_adjacency_f1}
\end{figure}

\section{Real-World Evaluation on Causal Chamber}
\label{app:causal-chamber}

The synthetic evaluation is complemented with a real-world experiment using the Causal Chamber wind-tunnel dataset \citep{gamella2025causal}. The experiment uses the 16 independent actuator random walks from \texttt{wt\_walks\_v1}, processing and evaluating each run separately.

Each run contains time-series measurements of the same 16 variables, which constitute the nodes of the causal graph.  Within each trajectory, every variable is independently rescaled to $[0,1]$ by min--max normalization. This places measurements with different physical units on a comparable numerical scale and improves numerical stability. The same preprocessing is applied to all methods.

TS-BOSS is compared with PCMCI+, DYNOTEARS and TS-FGES. All methods use a maximum lag of $\tau_{\max}=10$; the remaining settings follow the default configurations summarized in Table~\ref{tab:default_params}.

The Causal Chamber provides a validated directed causal graph between variables, but it does not assign temporal lags to its edges. Exact lag recovery therefore cannot be evaluated. For evaluation, each estimated temporal relation is projected onto a directed static relation:
\[
    X_{t-\tau}^{i}\rightarrow X_{t}^{j}
    \quad\mapsto\quad
    X_i \rightarrow X_j.
\]
Relations estimated at multiple lags are counted once, and autoregressive links are excluded because they are not represented in the reference graph. The projected graphs are evaluated using the adjacency metrics described in Appendix~\ref{sec:appendix_metrics}.

Table~\ref{tab:causal-chamber-results} reports the mean and standard error across the 16 trajectories, with $\mathrm{SE}=s/\sqrt{16}$. TS-BOSS achieves the highest mean adjacency recall and F1 score, whereas TS-FGES obtains the highest precision and a closely comparable F1 score. Thus, TS-BOSS recovers a larger proportion of the validated physical causal relations, at the cost of a higher number of false-positive adjacencies. Because the reference graph contains no lag annotations, these results measure variable-level causal-structure recovery rather than exact temporal-lag recovery. The complete experiment is available in the \href{https://github.com/irenecastillo/TS-BOSS}{TS-BOSS repository}.

\begin{table}[htbp]
    \centering
    \begin{tabular}{lccc}
        \toprule
        Method & Precision & Recall & F1 score \\
        \midrule
        TS-BOSS
        & $0.311 \pm 0.008$
        & $\mathbf{0.630 \pm 0.016}$
        & $\mathbf{0.416 \pm 0.010}$ \\
        TS-FGES
        & $\mathbf{0.395 \pm 0.021}$
        & $0.433 \pm 0.024$
        & $0.410 \pm 0.021$ \\
        PCMCI+
        & $0.318 \pm 0.020$
        & $0.276 \pm 0.016$
        & $0.293 \pm 0.015$ \\
        DYNOTEARS
        & $0.253 \pm 0.015$
        & $0.113 \pm 0.014$
        & $0.153 \pm 0.014$ \\
        \bottomrule
    \end{tabular}
    \caption{%
       Static adjacency recovery on the Causal Chamber wind-tunnel dataset. Values are mean $\pm$ standard error across 16 independent runs.
    }
    \label{tab:causal-chamber-results}
\end{table}

% \begin{table}[htbp]
%     \centering
%     \caption{Static adjacency recovery on the Causal Chamber wind-tunnel
%     dataset. Values are mean $\pm$ standard error across 16 independent runs.}
%     \label{tab:causal-chamber-results_time}
%     \begin{tabular}{lcccc}
%         \toprule
%         Method & Precision & Recall & F1 score & Runtime (s) \\
%         \midrule
%         TS-BOSS
%         & $0.311 \pm 0.008$
%         & $\mathbf{0.630 \pm 0.016}$
%         & $\mathbf{0.416 \pm 0.010}$
%         & $394.438 \pm 23.896$ \\
%         TS-FGES
%         & $\mathbf{0.395 \pm 0.021}$
%         & $0.433 \pm 0.024$
%         & $0.410 \pm 0.021$
%         & $4.516 \pm 0.159$ \\
%         PCMCI+
%         & $0.318 \pm 0.020$
%         & $0.276 \pm 0.016$
%         & $0.293 \pm 0.015$
%         & $68.981 \pm 11.500$ \\
%         DYNOTEARS
%         & $0.253 \pm 0.015$
%         & $0.113 \pm 0.014$
%         & $0.153 \pm 0.014$
%         & $\mathbf{0.794 \pm 0.112}$ \\
%         \bottomrule
%     \end{tabular}
% \end{table}

\section{Memory Scalability}
\label{app:memory}

We further evaluate memory consumption as the number of nodes increases, while keeping the remaining experimental settings fixed according to the default configuration in Table~\ref{tab:default_params}. For each value of $N$, we generate 10 independent causal graphs. Within each repetition, all time-series methods are evaluated on the same data realization, while TS-BOSS (i.i.d.) uses the corresponding i.i.d.\ realization generated from the same causal model.

Each method is executed in an isolated process, and memory consumption is
measured as its peak resident set size (RSS). This includes dependencies,
native allocations, and the JVM used by TS-FGES. The experiments were
conducted on an Apple M5 Pro system with 48\,GB of RAM, using
Python~3.11.16 and OpenJDK~26.0.2.

As shown in Figure~\ref{fig:memory_scalability}, TS-BOSS maintains the lowest
peak memory consumption among the methods across the evaluated
dimensions. At $N=100$, TS-BOSS requires $111.2$\,MiB, compared with
$170.7$\,MiB for PCMCI+, $198.5$\,MiB for DYNOTEARS, and $706.7$\,MiB for
TS-FGES. Overall, the peak memory consumption of TS-BOSS increases moderately over the evaluated range of node counts.

\begin{figure}[htbp]
    \centering
    \includegraphics[width=0.95\linewidth]
    {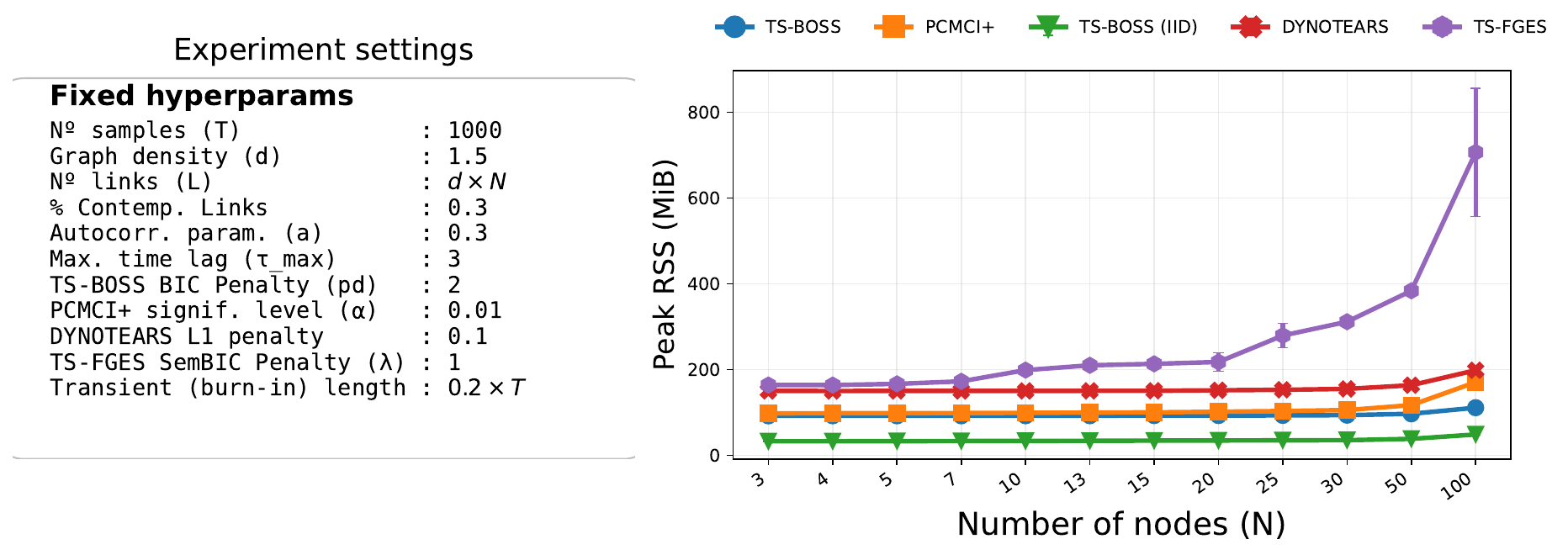}
    \caption{Memory consumption as the number of nodes increases. Points
    represent the mean peak RSS over 10 independent causal graphs, and error
    bars denote $\pm$ one sample standard deviation.}
    \label{fig:memory_scalability}
\end{figure}
\section{Grow and Shrink Algorithm Overview}\label{app:growshrink}

The following figures illustrate the grow and shrink phases of the GST procedure used in BOSS \citep{BOSS} and inherited by TS-BOSS.

\begin{figure}[htbp]
    \centering
    \includegraphics[width=0.95\linewidth]{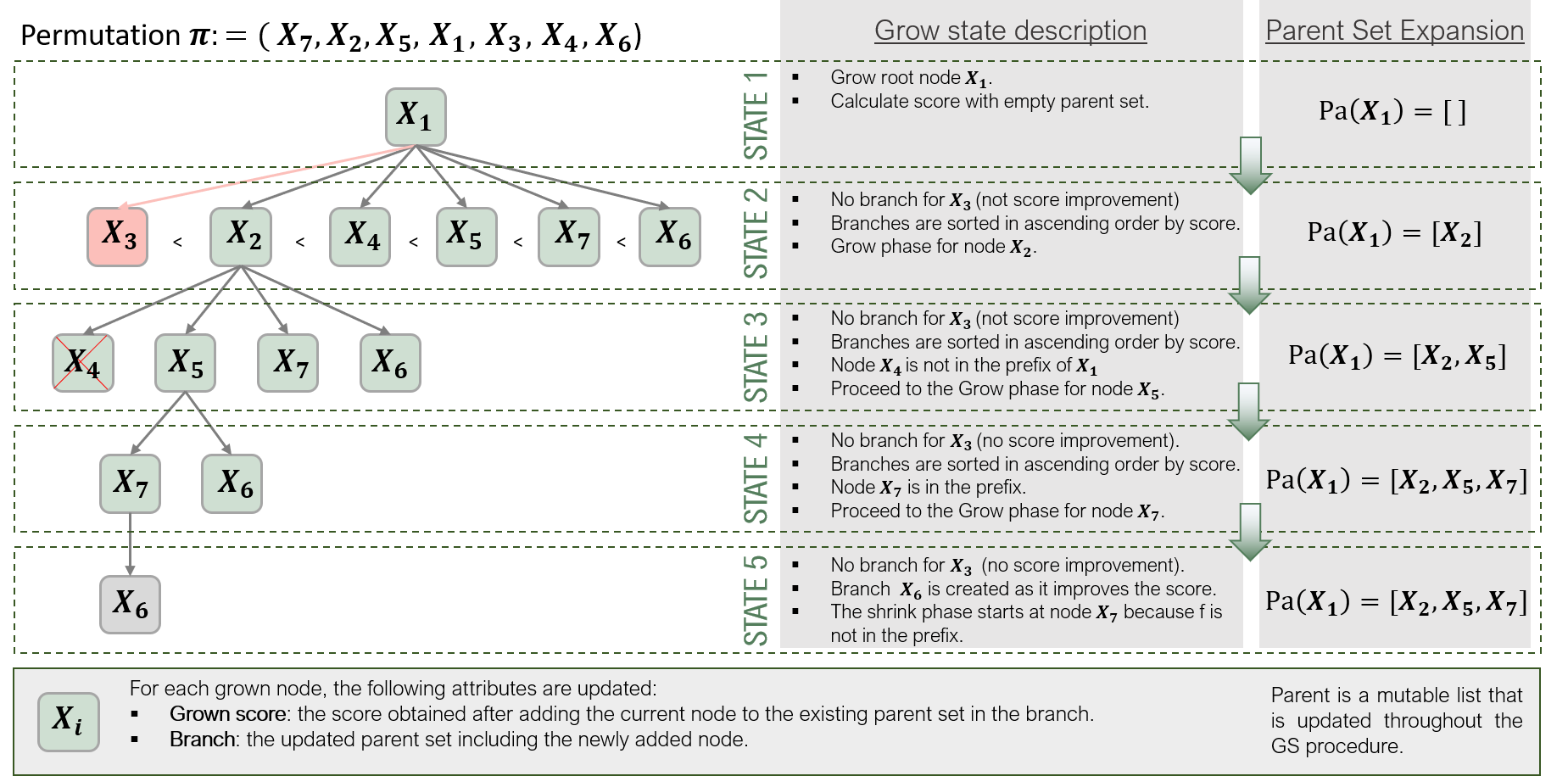}
    \caption{An overview of the \texttt{grow} algorithm presented in \citep{BOSS}. Given a target variable \(X_1\) and a permutation \(\pi\), the grow phase builds $X_1$'s candidate parent set by evaluating each available variable as a possible addition to the parent set and constructing branches on those that strictly improve \(X_1\)’s score. It then sorts these branches by the resulting (post-addition)
    grow score and chooses the best-ranked branch permissible by the order of $\pi$, that variable is added as a parent; the procedure continues until no further improving variable from the prefix of $\pi$ can  be added to $X_1$'s parent set.}
    % and chooses the best-ranked branch whose variable lies in the prefix set \(\mathrm{pre}_{\pi}(X)\); that variable is added as a parent and the procedure recurses.}
    \label{fig:grow}
\end{figure}

\begin{figure}[htbp]
    \centering
    \includegraphics[width=0.95\linewidth]{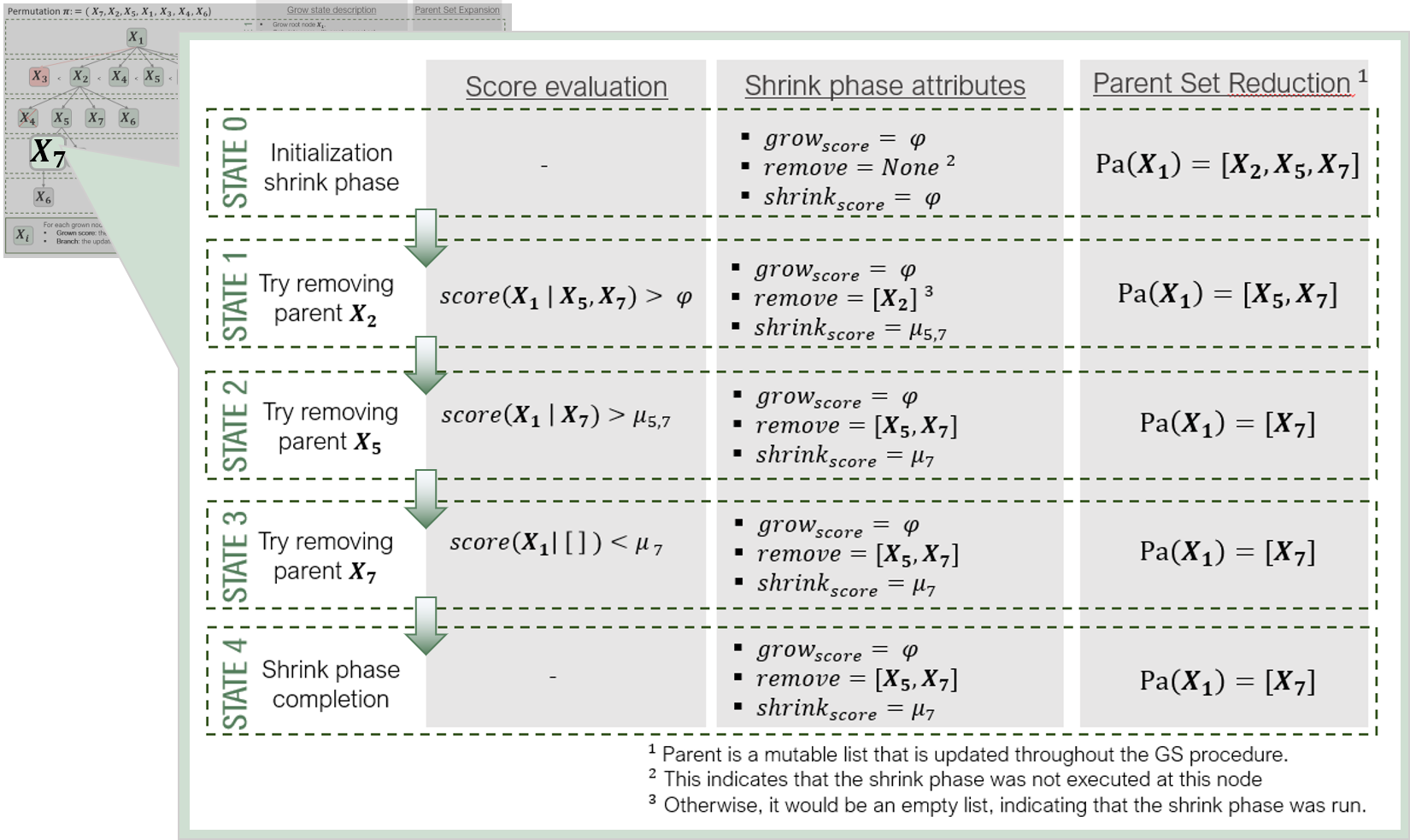}
    \caption{An overview of the \texttt{shrink} algorithm presented in \citep{BOSS}. After the grow phase, the parent set can include redundant variables, for instance, those that became redundant once other parents were added. The shrink phase prunes this set by iteratively removing any parent whose deletion improves the score, repeating until no further score-improving removals are possible.}
    \label{fig:shrink}
\end{figure}
% \clearpage
% \let\clearpage\relax
\end{document}